\documentclass[a4paper,10pt]{IEEEtran} \usepackage{amsmath,amssymb}
\usepackage{latexsym} \usepackage{graphicx} 
\usepackage{bm} \usepackage[latin1]{inputenc} \usepackage[T1]{fontenc}
\usepackage{algorithm}
\usepackage{algorithmic}
\usepackage{multicol}

\newtheorem{Step}{Step}
\newtheorem{theorem}{Theorem}
\newtheorem{lemma}{Lemma}
\newtheorem{definition}{Definition}

\begin{document}

\title{A Novel Model of Working Set Selection\\
       for SMO Decomposition Methods } 
\author{\authorblockN{Zhen-Dong Zhao$^{1}$} 
\and
\authorblockN{Lei Yuan$^{2}$} 
\and 
\authorblockN{Yu-Xuan Wang$^{2}$}
\and
\authorblockN{Forrest Sheng Bao$^{2}$}
\and
\authorblockN{Shun-Yi Zhang$^{1}$}
\and
\authorblockN{Yan-Fei Sun$^{1}$}
\\
\authorblockA{1.Institution of Information \& Network
      Technology, \\ Nanjing University of Posts and
      Telecommunications, Nanjing, 210003, CHINA}\\
\authorblockA{2.School of Communications and Information
      Engineering, \\ Nanjing University of Posts and
      Telecommunications, Nanjing, 210003, CHINA} \\
\authorblockA{\{zhaozhendong, lyuan0388, logpie, forrest.bao\}@gmail.com}
\and
\authorblockA{\{dirzsy, sunyanfei\}@njupt.edu.cn}}

\maketitle
\begin{abstract}
In the process of training Support Vector Machines~(SVMs) by
decomposition methods, working set selection is an important
technique, and some exciting schemes were employed into this field. To
improve working set selection, we propose a new model for working set
selection in sequential minimal optimization~(SMO) decomposition
methods. In this model, it selects $B$ as 
working set without reselection. Some properties are given by simple
proof, and experiments demonstrate that the proposed method is in
general faster than existing methods.
\end{abstract}

\begin{keywords}
support vector machines, decomposition methods, sequential minimal
optimization, working set selection
\end{keywords}

\IEEEpeerreviewmaketitle
\thispagestyle{empty}

\section{Introduction}
In the past few years, there has been huge of interest
in Support Vector Machines~(SVMs)~\cite{Boser.1992, Cortes.1995} because
they have excellent generalization performance on a wide range of
problems. The key work in training SVMs is to
solve the follow quadratic optimization problem.

\begin{align}
\min_{\alpha B} & f(\alpha) = \frac{1}{2} \alpha ^{T} Q \alpha - e^{T}
\alpha \notag \\ 
\quad\text{subject to }\quad & 0 \leq \alpha_{i} \leq
C, i=1,\dots,l \\ 
& y^{T} \alpha = 0 \notag
\end{align}
where $e$ is the vector of all ones, $C$ is the upper bound of all
variables, and $ Q_{ij}=y_{i} y_{j} K(x_{i},x_{j})$, $K(x_{i},x_{j})$ is the kernel function.

Notable effects have been taken into training SVMs~\cite{Chang.2001,
  Joachims.1998, Platt.1998, Osuna.1997}. Unlike most optimization
methods which update the whole vector $\alpha$ in each iteration, the
decomposition method modifies only a subset of $\alpha$ per iteration.
In each iteration, the variable indices are split into a "working
set": $B \subseteq \{1,\dots,l\}$ and its complement $N=\{1,\dots,l\}
\setminus B$. Then, the sub-problem with variables $x_{i},i \in B$, is
solved, thereby, leaving the values of the remaining variables
$x_{j},j \in N$ unchanged. This method leads to a small sub-problem to
be minimized in each iteration. An extreme case is the Sequential
Minimal Optimization~(SMO)~\cite{Platt.1998, PlattUS.1999}, which
restricts working set to have only two elements. Comparative tests
against other algorithms, done by Platt~\cite{Platt.1998}, indicates
that SMO is often much faster and has better scaling properties.
 
Since only few components are updated per iteration, for difficult
problems, the decomposition method suffers from slow convergence.
Better method of working set selection can reduce the number of
iterations and hence is an important research issue. Some methods
were proposed to solve this problem and to reduce the time of training
SVMs~\cite{Fan.2005}. In this paper, 
we propose a new model to select the working set. In this model,
specially, it selects $B$ without reselection.
In another word, once $\{\alpha_{i},\alpha_{j}\} \subset B$ are
selected, they will not be tested or selected during the following
working set selection. Experiments demonstrate that the new model is
in general faster than existing methods.

This paper is organized as following. In section II, we give literature
review, SMO decomposition method and existing working set
selection are both discussed. A new method of working set selection is then
presented in section III. In section IV, experiments with corresponding
analysis are given. Finally, section V concludes this paper.

\section{Literature Review}
In this section we discuss SMO and existing working set selections.

\subsection{Sequential Minimal Optimization}
Sequential Minimal Optimization~(SMO) was proposed by
Platt~\cite{Platt.1998}, which is an extreme case of the decomposition
algorithm where the size of working set is restricted to two. This
method is named as \textit{Algorithm 1}. Keerthi improved the
performance of this algorithm for training SVMs~(classifications and
regressions~\cite{Keerthi.2001, PlattUS.1999}). To take into account
the situation that the Kernel matrices are Non-Positive Definite,
Pai-Hsuen Chen \textit{et al.} ~\cite{Chen.2006, Fan.2005} introduce
the \textit{Algorithm 2} by restrict proof.

\textbf{Algorithm 2}
\begin{Step}
Find $\alpha^{1}$ as the initial feasible solution. Set k=1.
\end{Step}

\begin{Step}
If $\alpha^{k}$ is a stationary point of~(1), stop. Otherwise, find a
working set $B \equiv \{i,j\}$
\end{Step}

\begin{Step}
Let $a_{ij}=K_{ii}+K_{jj}-K_{ij}$.If $a_{ij}>0$, solve the
sub-problem. Otherwise, solve:

\begin{align}
\min_{B} \quad\text{ } & \frac{1}{2} \begin{bmatrix} \alpha_{i} &
  \alpha_{j} \end{bmatrix} \begin{bmatrix} Q_{ii} &Q_{ij} \\ Q_{ji} &
  Q_{jj}\end{bmatrix} \begin{bmatrix} \alpha_{i} \\
  \alpha_{j} \end{bmatrix} \notag \\ & + (-e_{B} +Q_{BN}
\alpha_{N}^{k})^{T} \begin{bmatrix} \alpha_{i} \\
  \alpha_{j} \end{bmatrix} \notag \\ & + \frac{\tau-\alpha_{ij}}{4}
((\alpha_{i} -\alpha_{i}^{k})^{2} + (\alpha_{j} - \alpha_{j}^{k})^{2})
\\ \quad\text{subject to }\quad & 0 \leq \alpha_{i} \leq C \notag \\ &
y_{i} \alpha_{i} +y_{j} \alpha_{j}=-y_{N}^{T} \alpha_{N}^{k} \notag
\end{align}
where $\tau$ is a small positive number.
\end{Step}

\begin{Step}
Set $\alpha^{+1}_{N} \equiv \alpha^{k}_{N}$. Set $k \leftarrow k+1$ and goto step 2.
\end{Step}

Working set selection is a very important procedure during training
SVMs. We discuss some existing methods in next subsection.

\subsection{Existing Working Set Selections}
Currently a popular way to select the working set $B$ is via the
"maximal violating pair", we call it \textit{WSS 1} for short. This
working set selection was first proposed in~\cite{Osuna.1997}, and is
used in, for example, the software LIBSVM~\cite{Chang.2001}. Instead
of using first order approximation which \textit{WSS 1} has used, a
new method which consider the more accurate second order information
be proposed~\cite{Fan.2005}, we call it \textit{WSS 2}. By using the
same $i$ as in \textit{WSS 1}, \textit{WSS 2} check only $O(l)$
possible $B$'s to decide j. Experiments indicate that a full check
does not reduce iterations of using \textit{WSS 2}
much~\cite{Chen.2006, Fan.2005}.

But,for the linear kernel, sometimes $K$ is only positive semi-definite, so
it is possible that $K_{ii}+K_{jj}-2K_{ij} =0$. Moreover, some existing
kernel functions~(e.g., sigmoid kernel) are not the inner product of
two vectors, so $K$ is even not positive semi-definite. Then
$K_{ii}+K_{jj}-2K_{ij} <0 $ may occur .

\setcounter{Step}{0}

For this reason, Chen \textit{et al.}~\cite{Fan.2005}, propose a new
working set selection named \textit{WSS 3}.

\textbf{WSS 3}
\begin{Step}
Select

\[
\overline{a_{ts}}  \equiv 
\begin{cases} a_{ts} & \quad\text{ if }\quad a_{ts} >0 \\
\tau & \quad\text{otherwise }\quad
\end{cases}
\]
\begin{align}
i  \in arg \max_{t} \{ & -y_{t} \triangledown f(\alpha^{k})_{t} | t \in I_{up}(\alpha^{k}) \},
\end{align}
\end{Step}

\begin{Step}
\begin{align}
\intertext{Consider Sub(B) defined and select}
j  \in arg \min_{t} \{ & -\frac{b_{it}^{2}}{\overline{a_{it}}} | t \in
I_{low} (\alpha^{k}),\notag \\
& -y_{t} \triangledown f(\alpha^{k})_{t} <  -y_{i} \triangledown f(\alpha^{k})_{i} \}
\end{align}
\end{Step}

\begin{Step}
Return $B=\{i,j\}$

where
 
\begin{align}
a_{ij} & \equiv K_{ii}+K_{jj}-2K_{ij} >0 \notag \\
b_{ij} & \equiv -y_{i} \triangledown f(\alpha^{k})_{i} + y_{j}
\triangledown f(\alpha^{k})_{j} >0 \notag
\end{align}
\end{Step}

\section{Our New Method}
\subsection{Interesting Phenomena}
Interestingly, when we test the datasets by
LIBSVM~\cite{Chang.2001} which employs \textit{Algorithm 2} and \textit{WSS 3}, some
phenomena attract us. Fig.~\ref{a1a_reused} illustrates two of them.

\begin{figure}[!htp]
\begin{center}
\includegraphics[scale=0.3]{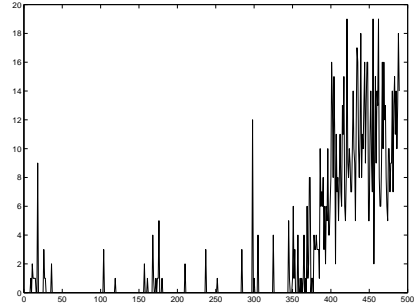}
\caption[a1a_reused]{The illustration of frequency of $\alpha$ reselection.
  Using dataset A1A, and kernel method RBF. Abscissa indicates the
  indices of $\alpha$, and ordinate indicates the number of times a
  certain $\alpha$ is picked in the training process}%
\label{a1a_reused}
\end{center}
\end{figure}

The first phenomenon is, lots of $\alpha$ have not been selected at
all in the training process. Because of the using of
"Shrinking"~\cite{Joachims.1998}, some samples will be "shrunk"
during the training procedure, thus, they will never be selected and
optimized. At the same time, we notice
another interesting phenomenon that several $\alpha$ are selected to
optimize the problem again and again, while others remain untouched.
But does this kind of reselection necessary?



\subsection{Working Set Selection Without Reselection(\textit{WSS-WR})}
After investigating the phenomena above, if we limit the
reselection of $\alpha$, we could effectively reduce the time for
training SVMs, with an acceptable effect on the ability of
generalization. Thus, we propose our new method of working set
selection where a certain $\alpha$ can only be selected once.

Before introducing the novel working set selection, we give some
definitions:

\begin{definition}
$T^{k+1}, k \in \{1,\cdots,\frac{l}{2} \}$ is denoted as optimized
set, in which $\forall \alpha \in T$ has been selected and optimized
once in working set selection.
\end{definition}

\begin{definition}
$C^{k+1} \subset \{ 1,\dots,l\} \setminus T^{k}$ is called available
set, in which $\forall \alpha \in C$ has never been selected.
\end{definition}

For optimization problem~(1), $\alpha \equiv C \cup T \cup B$. 

Our method can be described as following: In iteration $k$, a set $B^{k}
\equiv \{\alpha_{i}^{k},\alpha_{j}^{k}\}$ will be selected from the
new method which considers the more accurate second order
information available set $C^{k}$ by \textit{WSS 3}, after optimization,
both sample in this set will be put into $T^{k+1}$, that is to say,
$T^{k+1} = T^{k} \cup B^{k}$, $C^{k+1} = C^{k} \setminus B^{k}$. In
other words, once a working set is selected and optimized, all samples
in it will not be examined anymore in the following selection. The
relationship between sets $B, C, T$ are shown in the Fig.~\ref{model}
\begin{figure}[!htp]
\begin{center}
\includegraphics[scale=0.4]{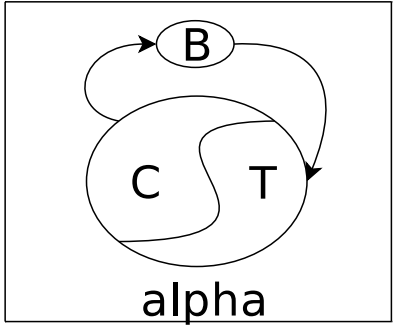}
\caption[model]{The Model of Working Set Selection Without Reselection~(\textit{WSS-WR})}%
\label{model}
\end{center}
\end{figure}

\setcounter{Step}{0}

\subsubsection{Working Set Selection Without Reselection~(WSS-WR)}
\begin{Step}
Select

\[
\overline{a_{ts}}  \equiv 
\begin{cases} a_{ts} & \quad\text{if}\quad a_{ts} >0 \\
\tau & \quad\text{otherwise}\quad
\end{cases}
\]

\begin{align}
i  \in arg \max_{t} \{ & -y_{t} \triangledown f(\alpha^{k})_{t} | t
\in I_{up}(\alpha^{k}) \cap t \not\in T^{k} \}, \\
j  \in arg \min_{t} \{ & -\frac{b_{it}^{2}}{\overline{a_{it}}} | t \in
I_{low}(\alpha^{k}) \cap t \not\in T^{k}, \notag \\ 
& -y_{t} \triangledown f(\alpha^{k})_{t} <  -y_{i} \triangledown f(\alpha^{k})_{i} \}.
\end{align}
\end{Step}

\begin{Step}
$T^{k+1} = T^{k} \cup B^{k}$, $C^{k+1} = C^{k} \setminus B^{k}$
\end{Step}

\begin{Step}
Return $B=\{i,j\}$
\end{Step}

where 
\begin{align}
a_{ij} & \equiv K_{ii}+K_{jj}-2K_{ij} >0 \notag \\
b_{ij} & \equiv -y_{i} \triangledown f(\alpha^{k})_{i} + y_{j}
\triangledown f(\alpha^{k})_{j} >0 \notag
\end{align}
We name this new method as Working Set Selection Without
Reselection~(\textit{WSS-WR}), in which all $\alpha$ will not be
reselected.

\subsection{Some Properties of \textit{WSS-WR} Model}
\textit{WSS-WR} has some special properties. In this
subsection, we simply prove some features of this new model.

\begin{theorem}
The values of all selected $\{\alpha_{i},\alpha_{j}\}
\subset B$ will always be 0.
\end{theorem}

\begin{proof}
Firstly, each $\alpha$ can only be selected once, which means once a
$\alpha$ is chosen for optimization, the value of it has never been
modified before; secondly, all $\alpha$ will be initialized as 0 at
the beginning of the algorithm. Thus, the values of all selected
$\{\alpha_{i},\alpha_{j}\} \subset B$ will be 0.
\end{proof}

\begin{theorem}
The algorithm terminates after a maximum of $\lceil
\frac{l}{2} \rceil$ iterations.
\end{theorem}

\begin{proof}
Firstly, the algorithm terminates if there is no sample left
in $C^{k}$ or certain optimization conditions are reached; secondly,
in each iteration, two samples will be selected and deleted from the
available set. Thus, under the worst situation, after $\lceil
\frac{l}{2} \rceil$ iterations, there will be no samples left in the
active set, the algorithm then terminates.
\end{proof}

\begin{lemma}
In the Model of \textit{WSS-WR}, $I_{up} \equiv I_{1}$ and $I_{low}
\equiv I_{4}$ .
\end{lemma}

\begin{proof}
According to the Theorem 1, the $\alpha \equiv 0$ which is chosen by
\textit{WSS-WR}. And Keerthi~\cite{Keerthi.2001} define the $I_{up},
I_{low}$ as:
\begin{align}
I_{0} \equiv \{i: 0 < \alpha_{i} < C\} \notag \\
I_{1} \equiv \{i: y_{i} = +1, \alpha_{i} = 0\} \\
I_{2} \equiv \{i: y_{i} = -1, \alpha_{i} = C\} \notag\\
I_{3} \equiv \{i: y_{i} = +1, \alpha_{i} = C\} \notag\\
I_{4} \equiv \{i: y_{i} = -1, \alpha_{i} = 0\} 
\end{align}

and $I_{up} \equiv \{I_{0} \cup I_{1} \cup I_{2}\}, I_{low} \equiv \{I_{0} \cup
I_{3} \cup I_{4}\}$ 

Thus $I_{up} \equiv I_{1}$ and $I_{low} \equiv
I_{4}$ in \textit{WSS-WR} model.
\end{proof}

\begin{lemma}
In the Model of \textit{WSS-WR}, $\alpha_{1}^{new} =
\alpha_{2}^{new}$.
\end{lemma}

\begin{proof}
During the SMO algorithm searches through the feasible region of the dual
problem and minimizes the objective function~(1).

Because $\sum_{i=1}^{N} y_{i} \alpha_{i} =0$, we have
\begin{align}
& y_{1}\alpha_{1}^{new} + y_{2} \alpha_{2}^{new} =
  y_{1}\alpha_{1}^{old} + y_{2} \alpha_{2}^{old} \notag\\ 
\intertext{Since}
& \alpha_{1}^{old} \in I_{up}, \alpha_{2}^{old} \in I_{low}
  \notag\\ 
\intertext{Therefore}
& y_{1}\alpha_{1}^{new} + y_{2} \alpha_{2}^{new}
  = 0 \notag \\ 
\Rightarrow & y_{1}\alpha_{1}^{new} = -y_{2} \alpha_{2}^{new}
  \notag \\ 
\intertext{According to the \textit{Lemma 1}}
& y_{1} = -y_{2} \notag \\ 
So \ &  \alpha_{1}^{new} = \alpha_{2}^{new} \notag
\end{align}
\end{proof}

\section{Computational Comparison}
In this section we compare the performance of our model against
\textit{WSS 3}. The comparison between \textit{WSS 1} and \textit{WSS
  3} have been done by Rong-En Fan \textit{et al.}~\cite{Fan.2005}.

\subsection {Data and Experimental Settings}
Some small datasets~(around 1,000 samples) including nine binary
classification and two regression problems are investigated under
various settings. And the large~(more than 30,000 instances)
classification problems are also taken into account. We select splice
from the Delve archive~(http://www.cs.toronto.edu/$\sim$delve). Problems
german.numer, heart and australian are from the Statlog collection
~\cite{Michie.1994}. Problems fourclass is from~\cite{TKH96a} and be
transformed to a two-class set. The datasets diabetes, breast-cancer,
and mpg are from the UCI machine learning repository
~\cite{Blake.1998}. The dataset mg from~\cite{GWF01a}. Problems a1a
and a9a are compiled in Platt~\cite{Platt.1998} from the UCI "adult"
dataset. Problems w1a and w8a are also from Platt~\cite{Platt.1998}.
The problem IJCNN1 is from the first problem of IJCNN 2001 challenge
~\cite{Prokhorov.2001}.

For most datasets, each attribute is linearly scaled to $[-1,1]$
except a1a, a9a, w1a, and w8a, because they take two values, 0 and 1.
All data are available at http://www.csie.ntu.edu.tw/
$\sim$cjlin/libsvmtools/. We set $ \tau = 10^{-12}$ both in \textit{WSS 3} and
\textit{WSS-WR}.

Because different SVMs parameters and kernel parameters affect the
training time, it is difficult to evaluate the two methods under
every parameter setting. For a fair comparison, we use the
experimental procedure which Rong-En Fan \textit{et al.}~\cite{Fan.2005} used:

1. "Parameter selection" step: Conduct five-fold cross validation to
find the best one within a given set of parameters.

2. "Final training" step: Train the whole set with the best parameter
to obtain the final model.

Since we concern the performance of using different kernels, we
thoroughly test four commonly used kernels:

1. RBF kernel:
\[
K(x_{i},x_{j})=e^{-\gamma \Vert x_{i}-x_{j} \Vert ^{2}}
\]

2. Linear kernel:
\[
K(x_{i},x_{j})=x_{i}^{T} x_{j}
\]

3. Polynomial kernel:
\[
K(x_{i},x_{j})=(\gamma (x_{i}^{T} x_{j}+1))^{d}
\]

4. Sigmoid kernel:
\[
K(x_{i},x_{j})= tanh (\gamma x_{i}^{T} x_{j} +d)
\]

\begin{table}
\centering
\begin{tabular}[!htp]{|c|c|c|c|}
\hline Kernel & Problems Type & $log_{2}^{C}$  & $log_{2}^{\gamma}$ \\
\hline RBF & Classification & -5,15,2 & 3,-15,-2   \\
\hline
           & Regression & -1,15,2 & 3,-15,-2 \\
\hline Linear & Classification & -3,5,2 &  \\
\hline
              & Regression & -3,5,2 &  \\
\hline Polynomial & Classification & -3,5,2 & -5,-1,1 \\
\hline
                  & Regression & -3,5,2 & -5,-1,1  \\
\hline Sigmoid & Classification & -3,12,3 & -12,3,3  \\
\hline
               & Regression & $\gamma= \frac{1}{\# Features}$ & -8,-1,3 \\
\hline
\end{tabular}
\caption{ Parameters used for various kernels: values of each parameter are from a uniform
discretization of an interval. We list the left, right end points and the space for
discretization. For example:-5,15,2. for $\log 2^{C}$ means $\log 2^{C} =
{-5,-3,\dots,15}$} \label{parameter_settings}
\end{table}

Parameters used for each kernel are listed in
Table~\ref{parameter_settings}. It is important to check how \textit{WSS-WR}
performs after incorporating shrinking and caching strategies. We
consider various settings:

1. With or without shrinking. 

2. Different cache size: First a 100MB cache allows the whole kernel
matrix to be stored in the computer memory. Second, we allocate only
100KB memory, so cache miss may happen and more kernel evaluations are
needed. The second setting simulates the training of large-scale sets
whose kernel matrices cannot be stored.

\subsection{Numerical Experiments}
\subsubsection{Comparison of Cross Validation Accuracy of Classification}
First, the grid method is applied. Cross validation accuracy is
compared during "parameters selection" and "final training".

\begin{table}
\begin{center}
\begin{tabular}[!htp]{|c|c|c|c|c|c|}
\hline method & a1a & w1a & aust. & spli. & brea. \\
\hline \textit{WSS-WR} & 83.4268 & 97.7796 & 86.2319 & 86 & 97.3646 \\
\hline \textit{WSS 3} & 83.8006 & 97.9814 & 86.8116 & 86.8 & 97.2182 \\
\hline
\hline method & diab. & four. & germ. & heart &\\
\hline \textit{WSS-WR} & 77.9948 & 100 & 75.8 & 85.1852 & \\
\hline \textit{WSS 3} &  77.474 & 100 & 77.6 & 84.4444 & \\
\hline
\end{tabular}
\caption{ Accuracy comparison between \textit{WSS 3} and \textit{WSS-WR}~(RBF) } \label{RBF_accuracy_100M_with}
\end{center}
\end{table}

\begin{table}
\begin{center}
\begin{tabular}[!htp]{|c|c|c|c|c|c|}
\hline method & a1a & w1a & aust. & spli. & brea. \\
\hline \textit{WSS-WR} & 83.053 & 97.4162 & 85.5072 & 80.4 & 97.2182 \\
\hline \textit{WSS 3} & 83.8629 & 97.8199 & 85.5072 & 79.5 & 97.2182 \\
\hline
\hline method & diab. & four. & germ. & heart &\\
\hline \textit{WSS-WR} & 76.4323 & 77.8422 & 75.3 & 83.7037 & \\
\hline \textit{WSS 3} &  77.2135 & 77.6102 & 77.9 & 84.4444 &\\
\hline
\end{tabular}
\caption{ Accuracy comparison between \textit{WSS 3} and \textit{WSS-WR}~(Linear) } \label{Linear_accuracy_100M_with}
\end{center}
\end{table}

\begin{table}
\begin{center}
\begin{tabular}[!htp]{|c|c|c|c|c|c|}
\hline method & a1a & w1a & aust. & spli. & brea. \\
\hline \textit{WSS-WR} & 83.3645 & 97.8603 & 86.087 & 82.1 & 97.6574 \\
\hline \textit{WSS 3} & 83.3645 & 97.6181 & 85.7971 & 82.7 & 97.511 \\
\hline
\hline method & diab. & four. & germ. & heart &\\
\hline \textit{WSS-WR} & 76.6927 & 78.6543 & 75.1 & 84.0741 &\\
\hline \textit{WSS 3} & 77.0833 & 79.6984 & 75.2 & 84.8148 &\\
\hline
\end{tabular}
\caption{ Accuracy comparison between \textit{WSS 3} and \textit{WSS-WR}~(Polynomial) } \label{Polynomial_accuracy_100M_with}
\end{center}
\end{table}

\begin{table}
\begin{center}
\begin{tabular}[!htp]{|c|c|c|c|c|c|}
\hline method & a1a & w1a & aust. & spli. & brea. \\
\hline \textit{WSS-WR} & 83.6137 & 97.6988 & 85.7971 & 80.4 & 97.0717\\
\hline \textit{WSS 3} & 83.4268 & 97.8603 & 85.6522 & 80.5 &  97.2182 \\
\hline
\hline method & diab. & four. & germ. & heart &\\
\hline \textit{WSS-WR}  & 77.2135 & 77.8422 & 75.8 & 84.8148 & \\
\hline \textit{WSS 3} & 77.2135 & 77.8422 & 77.8 & 84.0744 & \\
\hline
\end{tabular}
\caption{ Accuracy comparison between \textit{WSS 3} and \textit{WSS-WR}~(Sigmoid) } \label{Sigmoid_accuracy_100M_with}
\end{center}
\end{table}

We test various situations concerning all the commonly used kernels, with
and without shrinking technique and 100MB/100KB cache. The Table
\ref{RBF_accuracy_100M_with}-~\ref{Sigmoid_accuracy_100M_with} show
that cross validation accuracy between \textit{WSS-WR} and \textit{WSS
  3} method are almost the same. To be specifically, $|
Accuracy_{\textit{WSS-WR}}-Accuracy_{\textit{WSS 3}} | <0.026$. And in most of the
datasets, the accuracy of \textit{WSS 3} outperform that of
\textit{WSS-WR}. There are also several datasets in which
\textit{WSS-WR} performs even more accurate than \textit{WSS 3}.
Besides, great improvement are made both in the number of iterations
as well as the consumption of time.

\subsubsection{Comparison of Cross Validation Mean Squared Error of Regression}
\begin{table}
\begin{center}
\begin{tabular}[!htp]{|c|c|c|c|c|}
\hline Kernel Methods & \multicolumn{2}{|c|}{RBF} & \multicolumn{2}{|c|}{Linear} \\
\hline Method & \textit{WSS-WR} & \textit{WSS 3} & \textit{WSS-WR} & \textit{WSS 3}  \\
\hline MPG & 6.9927	& 6.4602 & 12.325 & 12.058\\
\hline MG & 0.01533 &	0.014618 & 0.02161 &	0.02138 \\
\hline Kernel Methods & \multicolumn{2}{|c|}{Polynomial} & \multicolumn{2}{|c|}{Sigmoid} \\
\hline Method & \textit{WSS-WR} & \textit{WSS 3} & \textit{WSS-WR} & \textit{WSS 3} \\
\hline MPG &  0.25 & 0.5 & 14.669 & 14.22\\
\hline MG &  0.019672 &	0.018778 & 0.023228 & 0.023548\\
\hline
\end{tabular}
\caption{ Cross validation mean squared error } \label{cross validation mean squared error}
\end{center}
\end{table}

The Table~\ref{cross validation mean squared error} shows that cross
validation mean squared error does not differ much in regression
between \textit{WSS-WR} and \textit{WSS 3}.

\subsubsection{Iteration and time ratios between \textit{WSS-WR} and \textit{WSS 3}}
After comparison of cross validation accuracy of classification and
cross validation mean squared error of regression, We illustrate the
iteration and time ratios between \textit{WSS-WR} and \textit{WSS 3}.

For each kernel, we give two figures showing results of "parameter
selection" and "final training" steps, respectively. We further
separate each figure to two situations: without/with shrinking, and
present five ratios between using \textit{WSS-WR} and using
\textit{WSS 3}:
\[
ratio 1 \equiv \frac{\quad\text{time of \textit{WSS-WR}~(100M cache,shrinking)}\quad}{\quad\text{time of \textit{WSS 3}~(100M cache,shrinking)}\quad}
\]
\[
ratio 2 \equiv \frac{\quad\text{time of \textit{WSS-WR}~(100M cache,nonshrinking)}\quad}{\quad\text{time of \textit{WSS 3}~(100M cache,nonshrinking)}\quad}
\]
\[
ratio 3 \equiv \frac{\quad\text{time of \textit{WSS-WR}~(100K cache,shrinking)}\quad}{\quad\text{time of \textit{WSS 3}~(100K cache,shrinking)}\quad}
\]
\[
ratio 4 \equiv \frac{\quad\text{time of \textit{WSS-WR}~(100K
    cache,nonshrinking)}\quad}{\quad\text{time of \textit{WSS 3}~(100K cache,nonshrinking)}\quad}
\]
\[
ratio 5 \equiv \frac{\quad\text{Total iteration of \textit{WSS-WR}}\quad}{\quad\text{Total iteration of \textit{WSS 3}}\quad}
\]

\begin{figure}[!htp]
\begin{center}
\includegraphics[width=.4\textwidth,height=.15\textwidth]{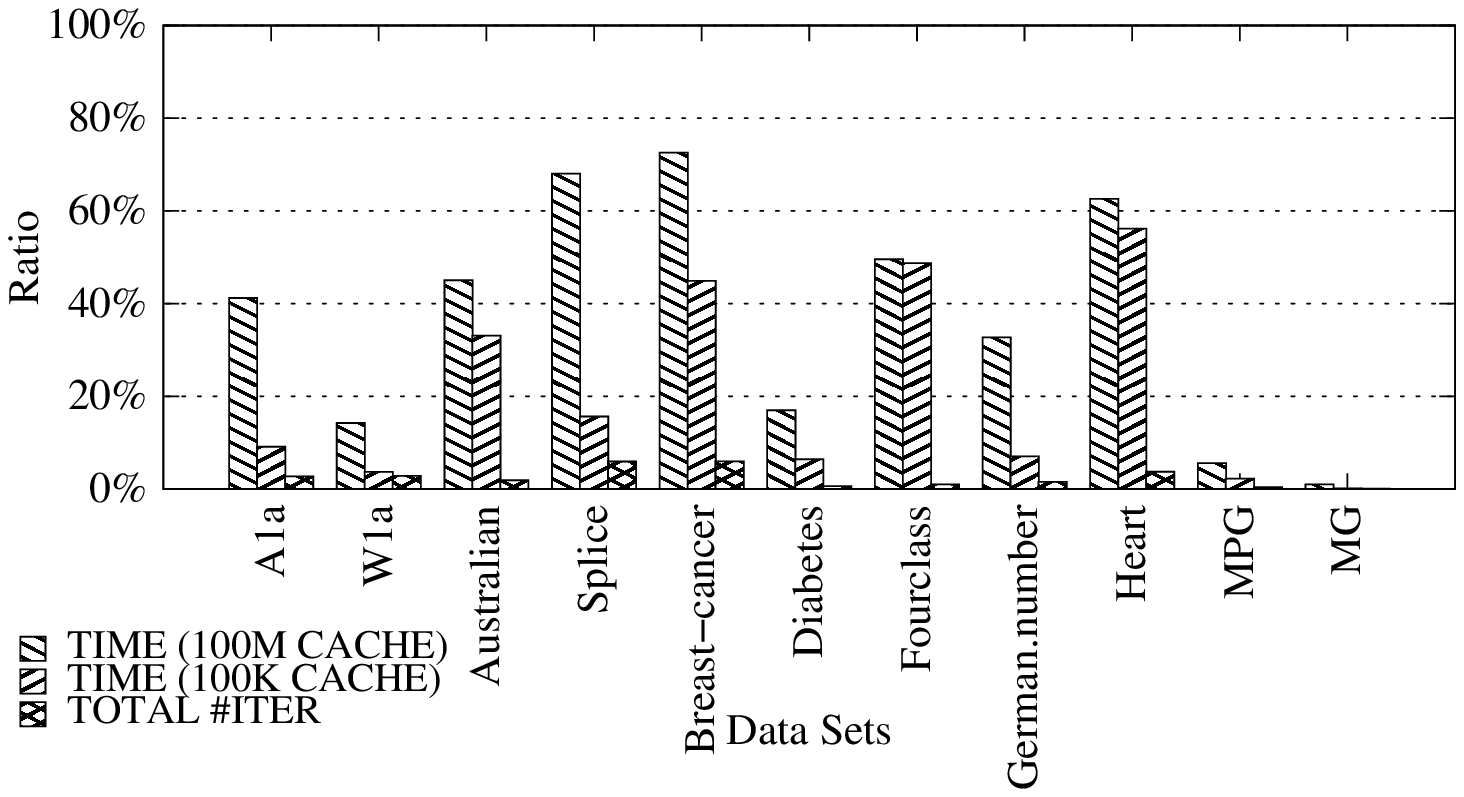}
\\
\includegraphics[width=.4\textwidth,height=.15\textwidth]{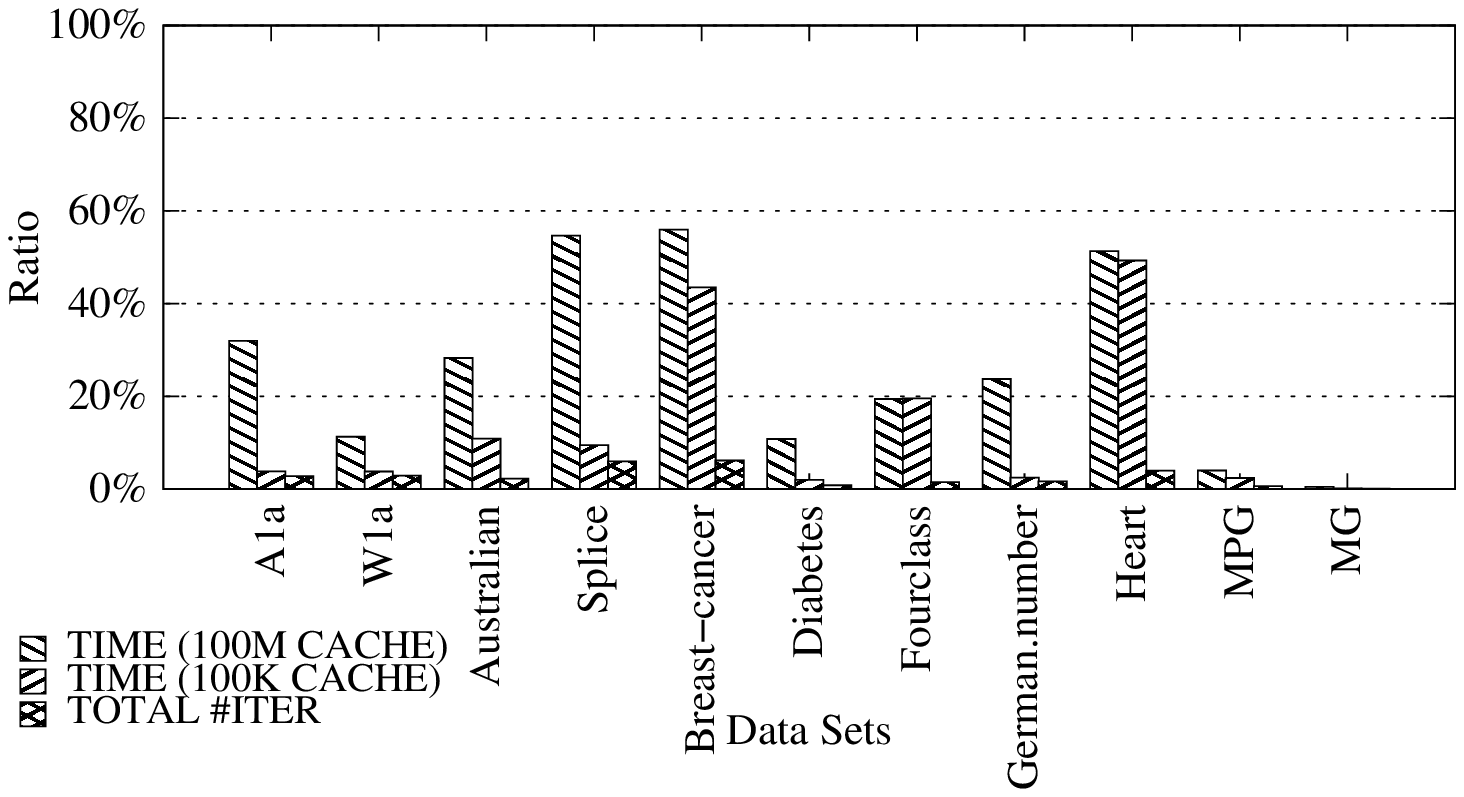}
\caption[rbf_ps]{Iteration and time ratios between \textit{WSS-WR} and \textit{WSS 3} using the RBF kernel for the "parameter selection" step~(top: with shrinking, bottom: without shrinking).}%
\label{RBF ps ratio}
\end{center}
\end{figure}

\begin{figure}[!htp]
\begin{center}
\includegraphics[width=.4\textwidth,height=.15\textwidth]{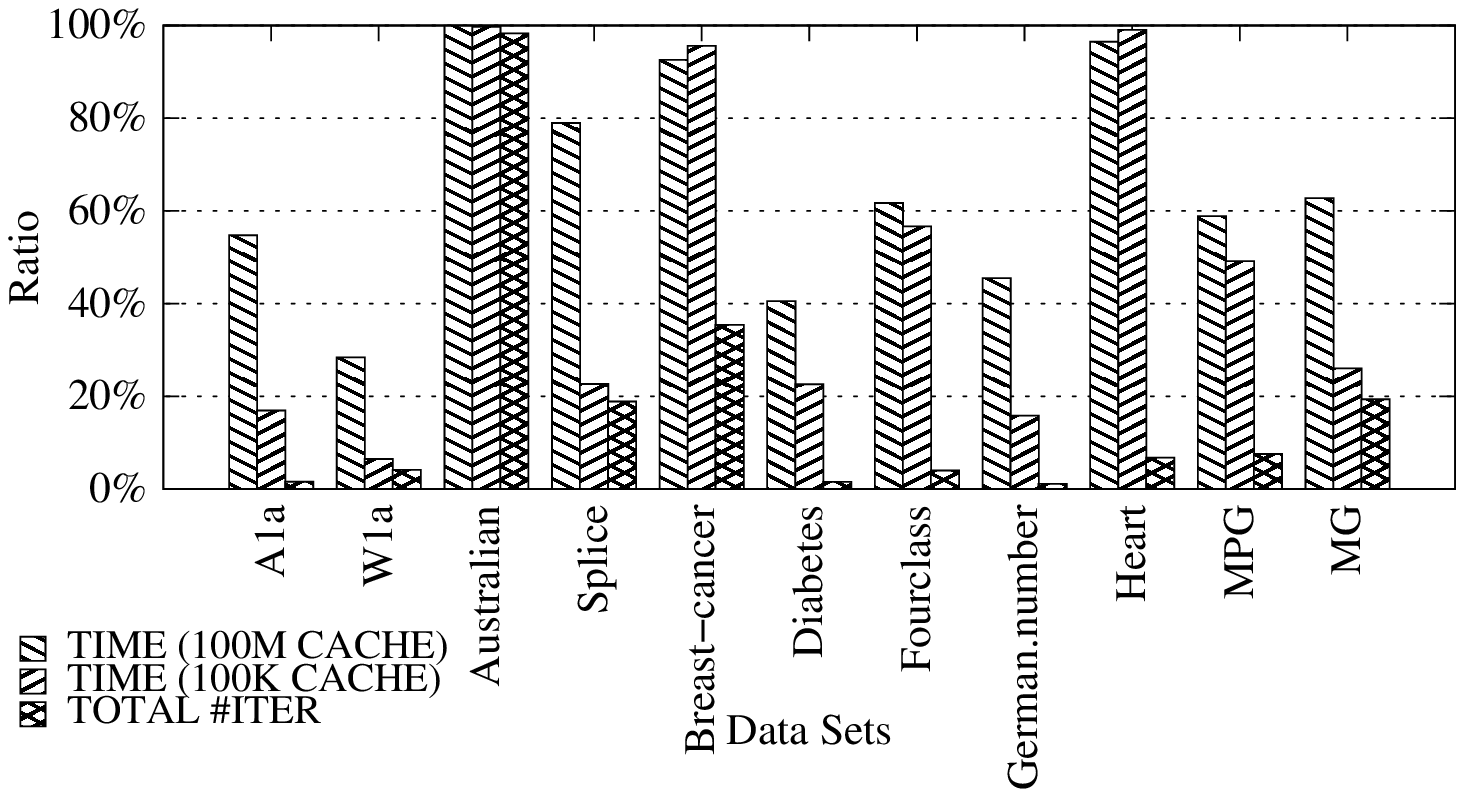} \\
\includegraphics[width=.4\textwidth,height=.15\textwidth]{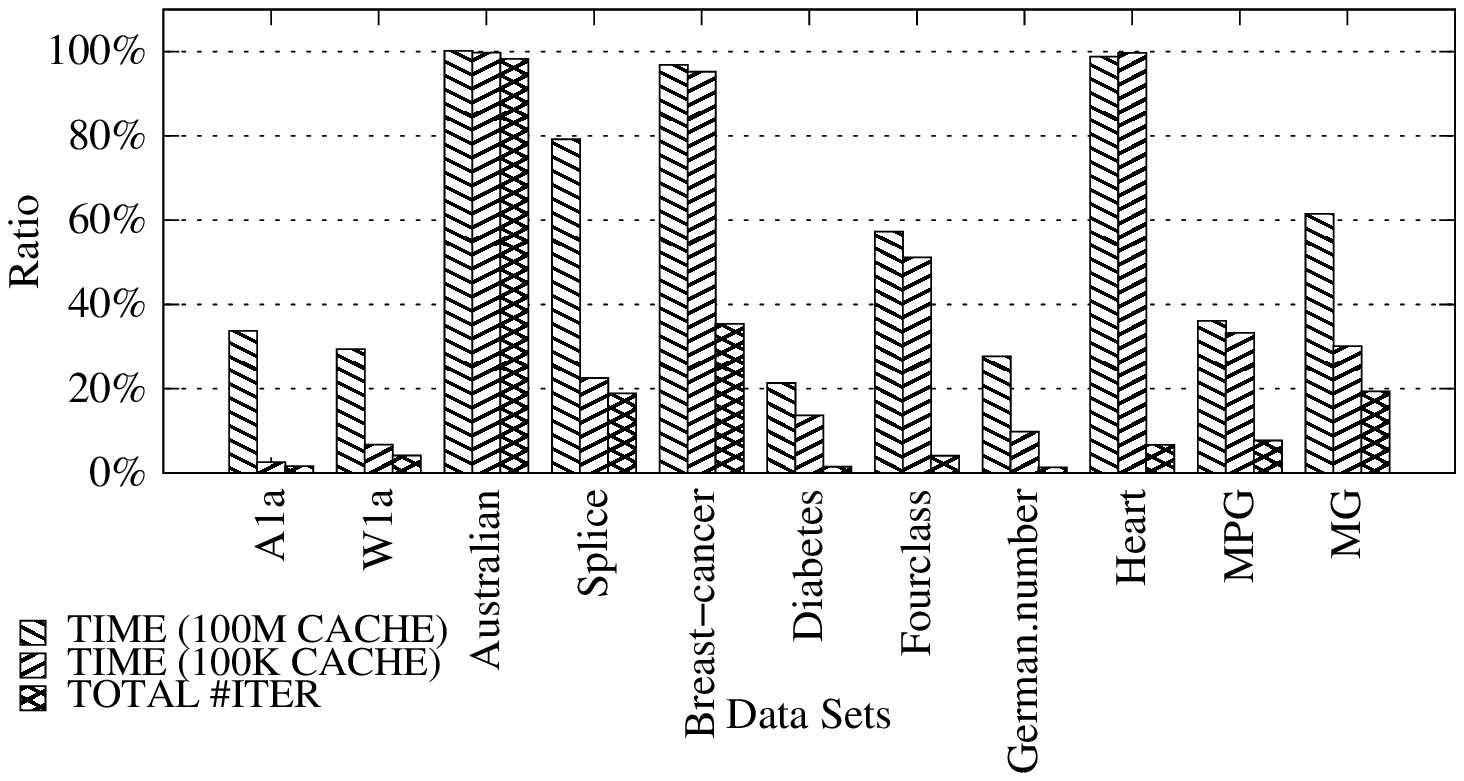}
\caption[rbf_ft]{Iteration and time ratios between \textit{WSS-WR} and \textit{WSS 3} using the RBF kernel for the "final training" step~(top: with shrinking, bottom: without shrinking).}%
\label{RBF ft ratio}
\end{center}
\end{figure}

\begin{figure}[!htp]
\begin{center}
\includegraphics[width=.4\textwidth,height=.15\textwidth]{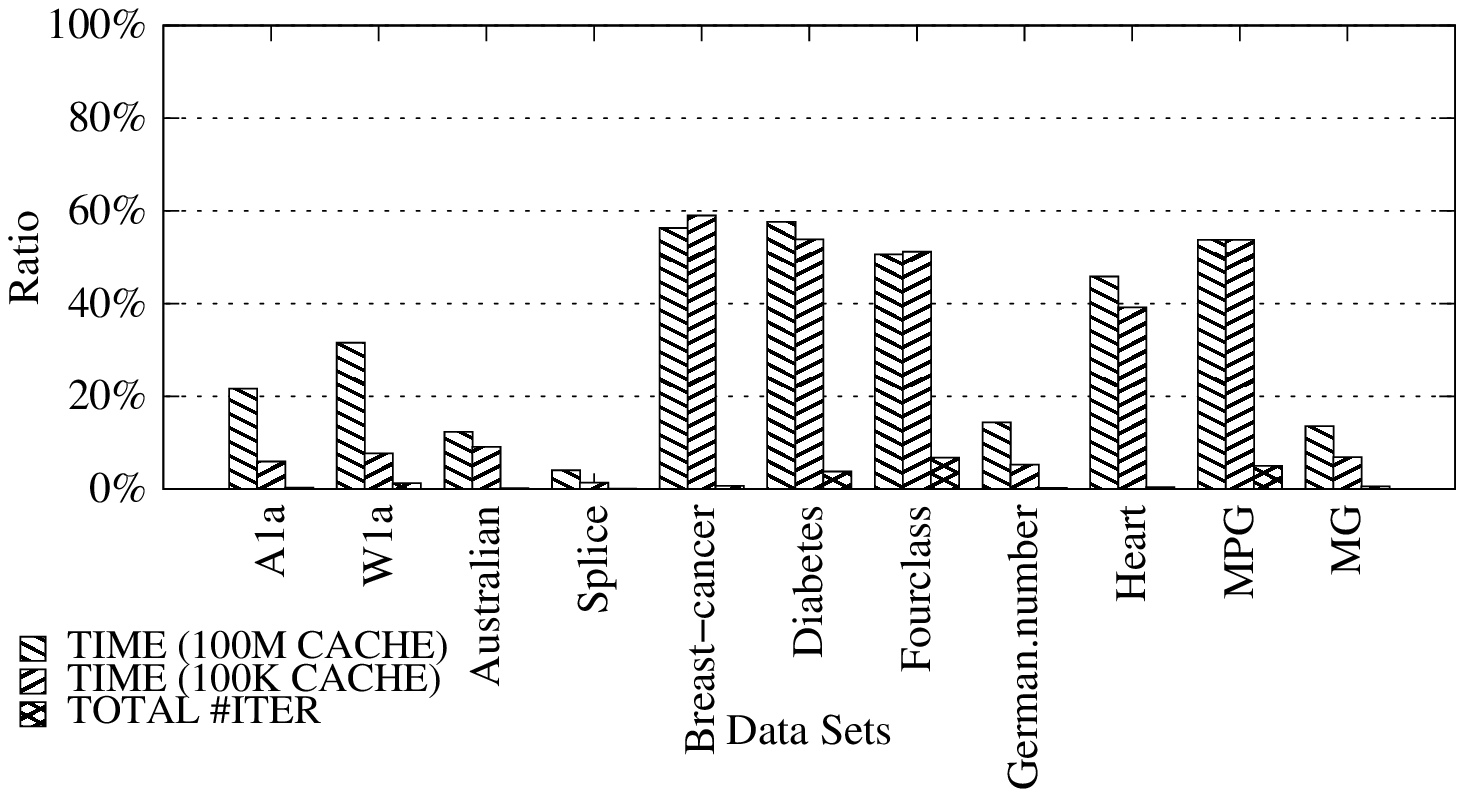} \\
 \includegraphics[width=.4\textwidth,height=.15\textwidth]{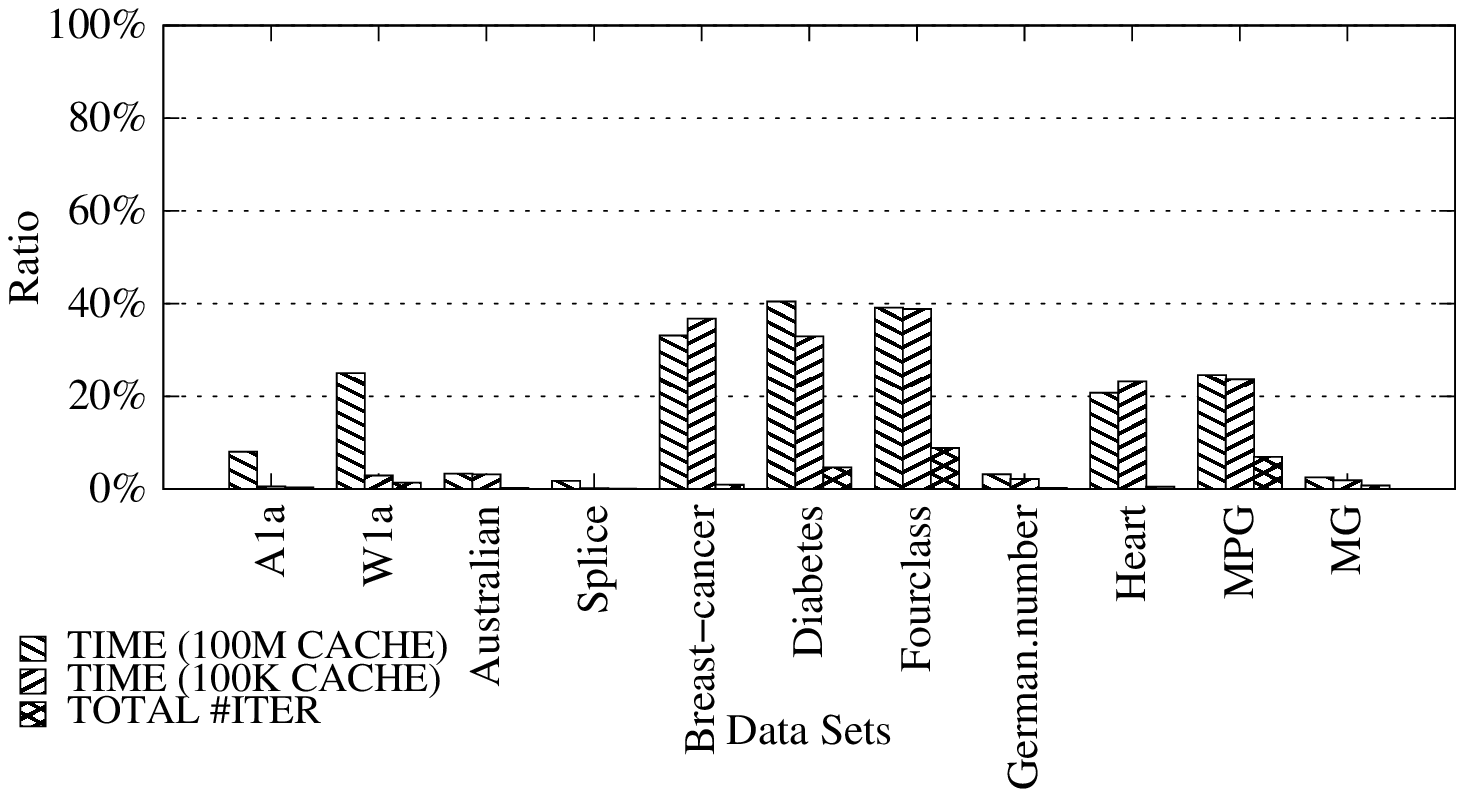}
\caption[Linear_ps]{Iteration and time ratios between \textit{WSS-WR} and \textit{WSS 3} using the Linear kernel for the "parameter selection" step~(top: with shrinking, bottom: without shrinking).}%
\label{Linear ps ratio}
\end{center}
\end{figure}

\begin{figure}[!htp]
\begin{center}nn
\includegraphics[width=.4\textwidth,height=.15\textwidth]{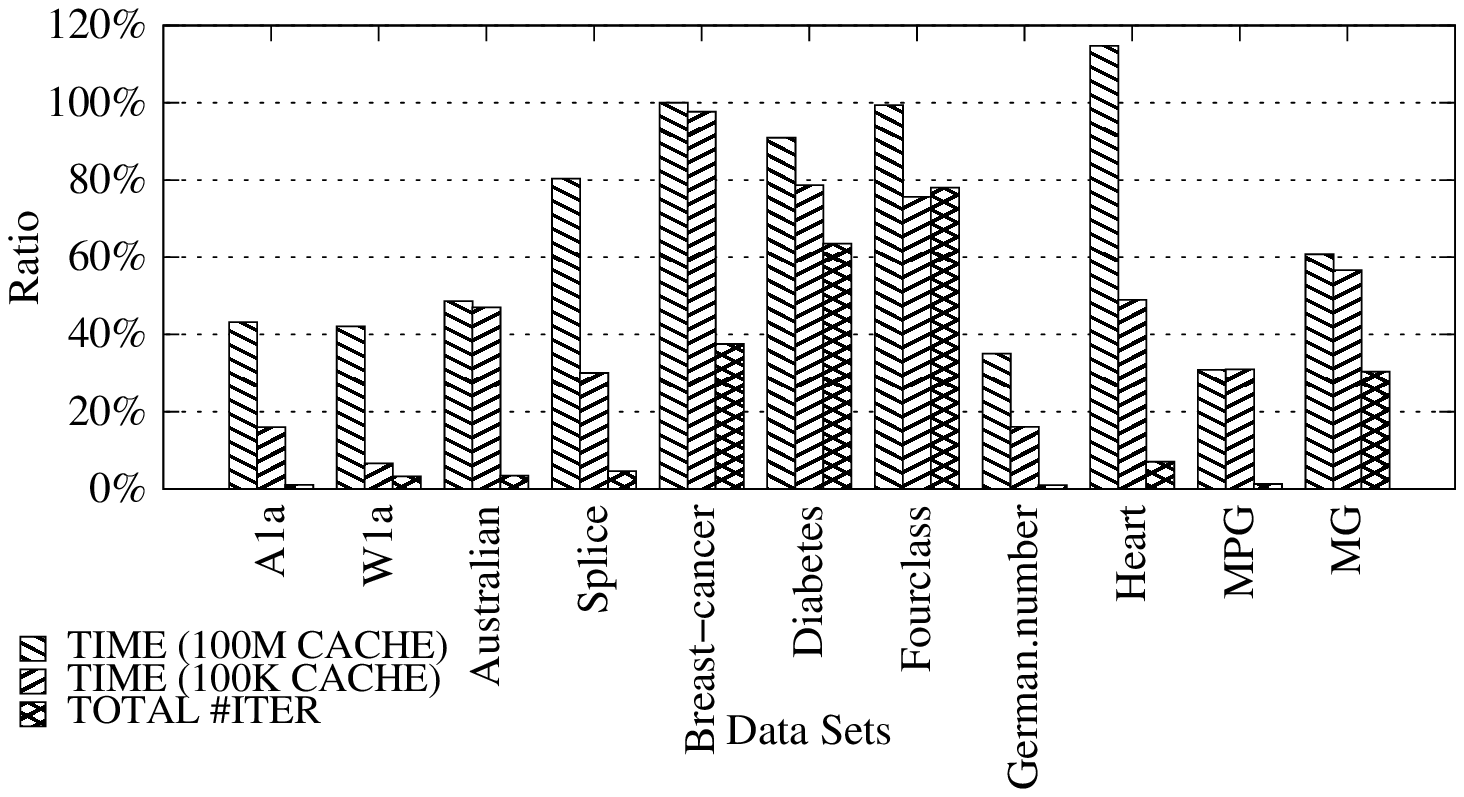} \\
\includegraphics[width=.4\textwidth,height=.15\textwidth]{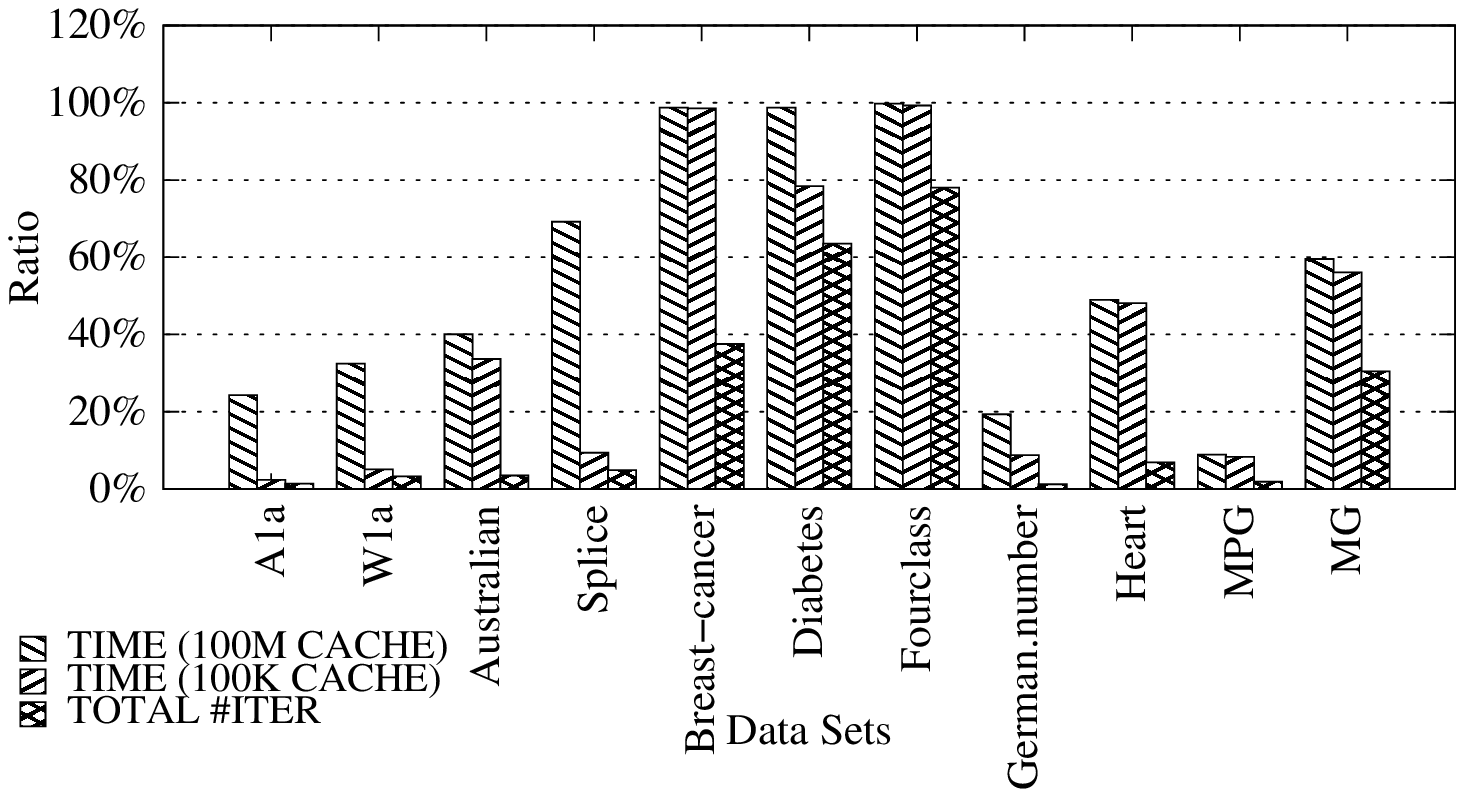}
\caption[Linear_ft]{Iteration and time ratios between \textit{WSS-WR} and \textit{WSS 3} using the Linear kernel for the "final training" step~(top: with shrinking, bottom: without shrinking).}%
\label{Linear ft ratio}
\end{center}
\end{figure}

\begin{figure}[!htp]
\begin{center}
\includegraphics[width=.4\textwidth,height=.15\textwidth]{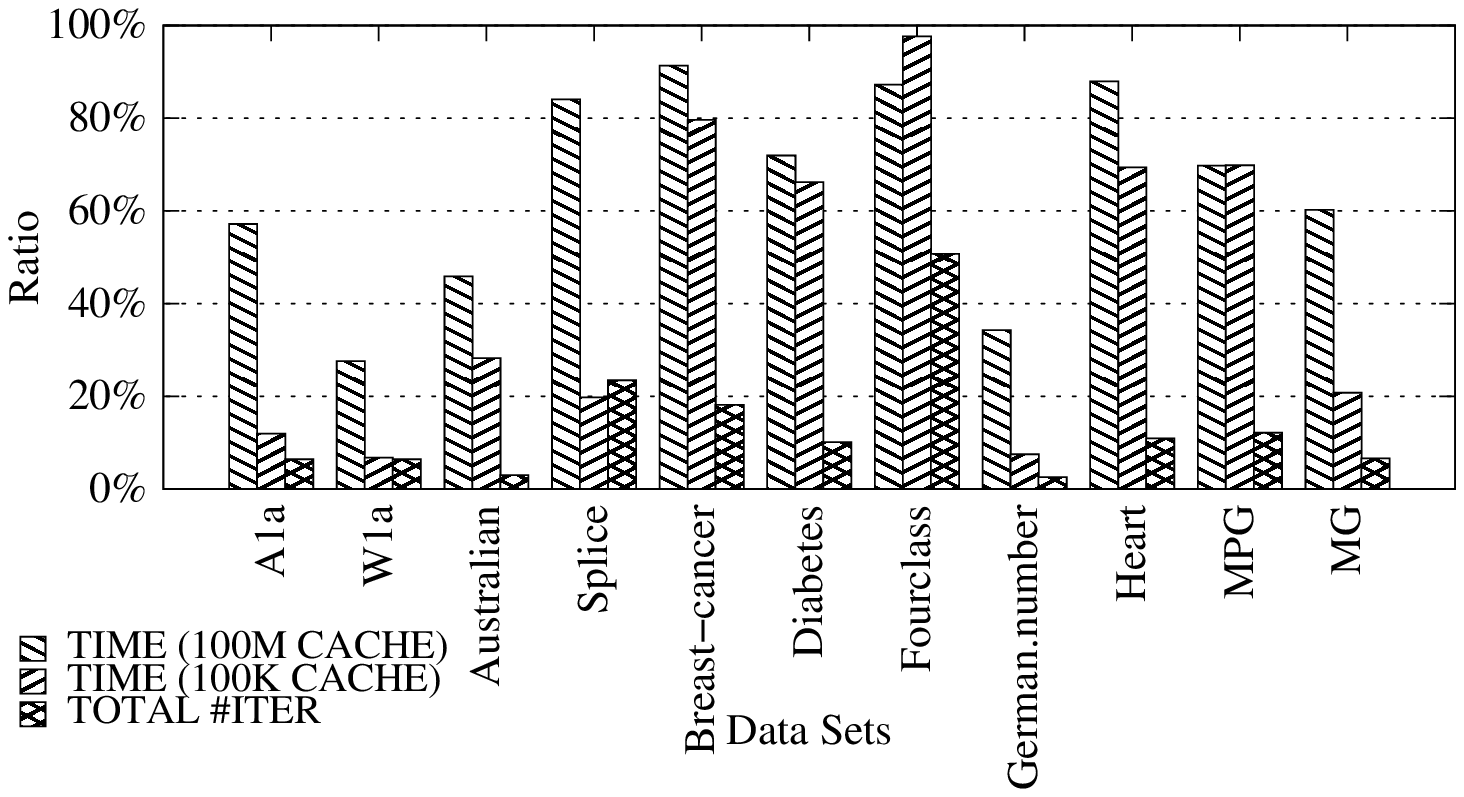} \\
 \includegraphics[width=.4\textwidth,height=.15\textwidth]{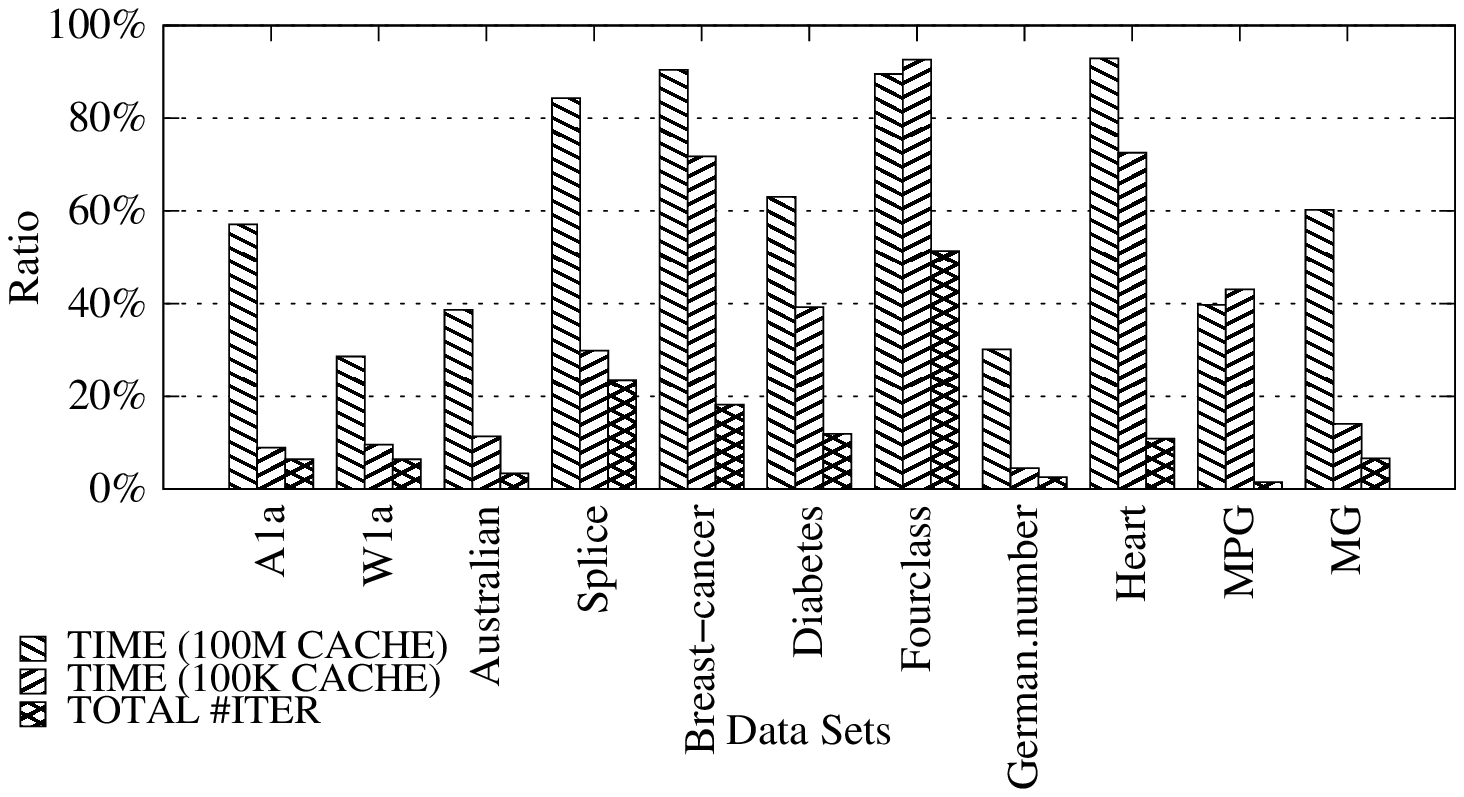}
\caption[Polynomial_ps]{Iteration and time ratios between \textit{WSS-WR} and \textit{WSS 3} using the Polynomial kernel for the "parameter selection" step~(top: with shrinking, bottom: without shrinking).}%
\label{Polynomial ps ratio}
\end{center}
\end{figure}

\begin{figure}[!htp]
\begin{center}
\includegraphics[width=.4\textwidth,height=.15\textwidth]{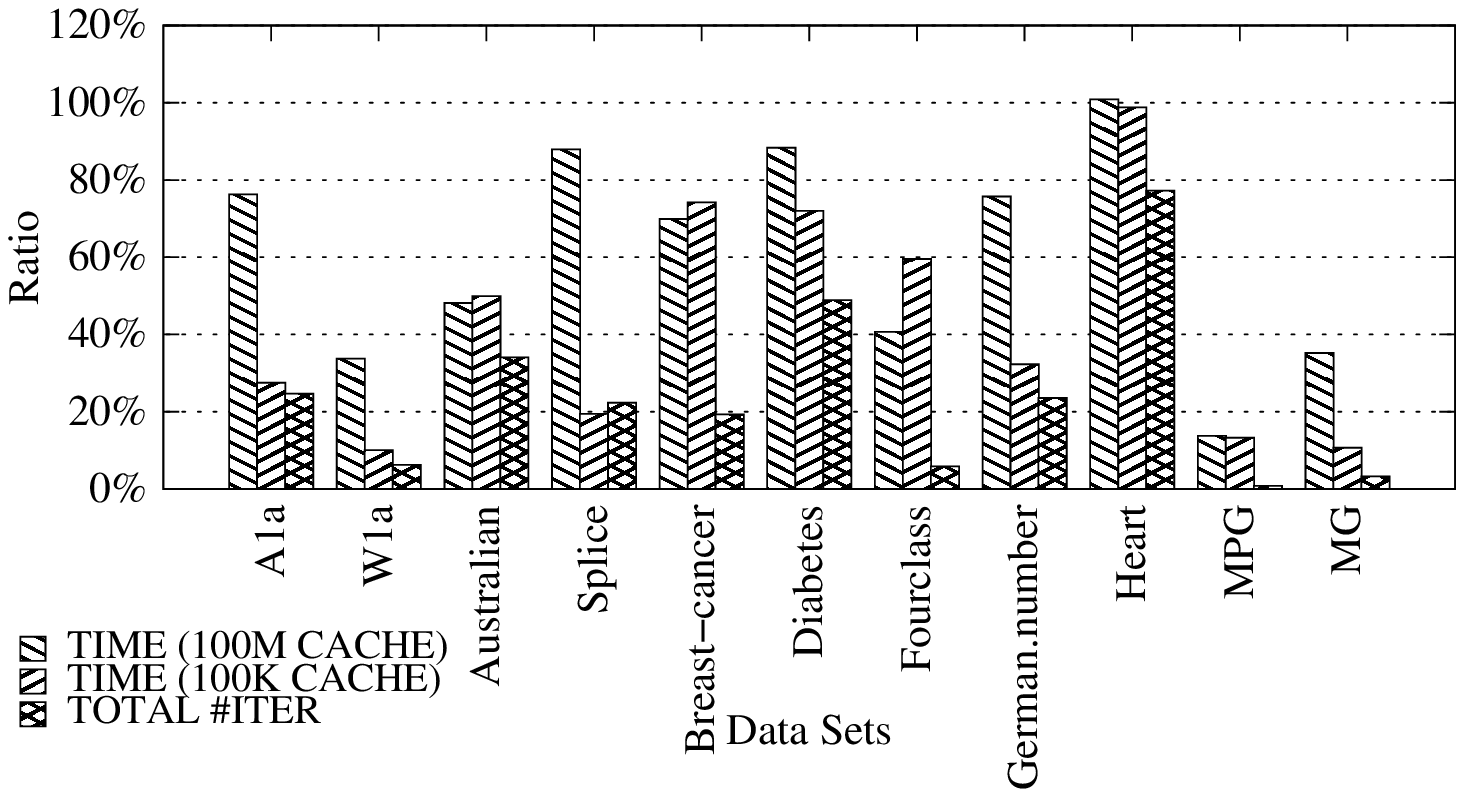} \\
\includegraphics[width=.4\textwidth,height=.15\textwidth]{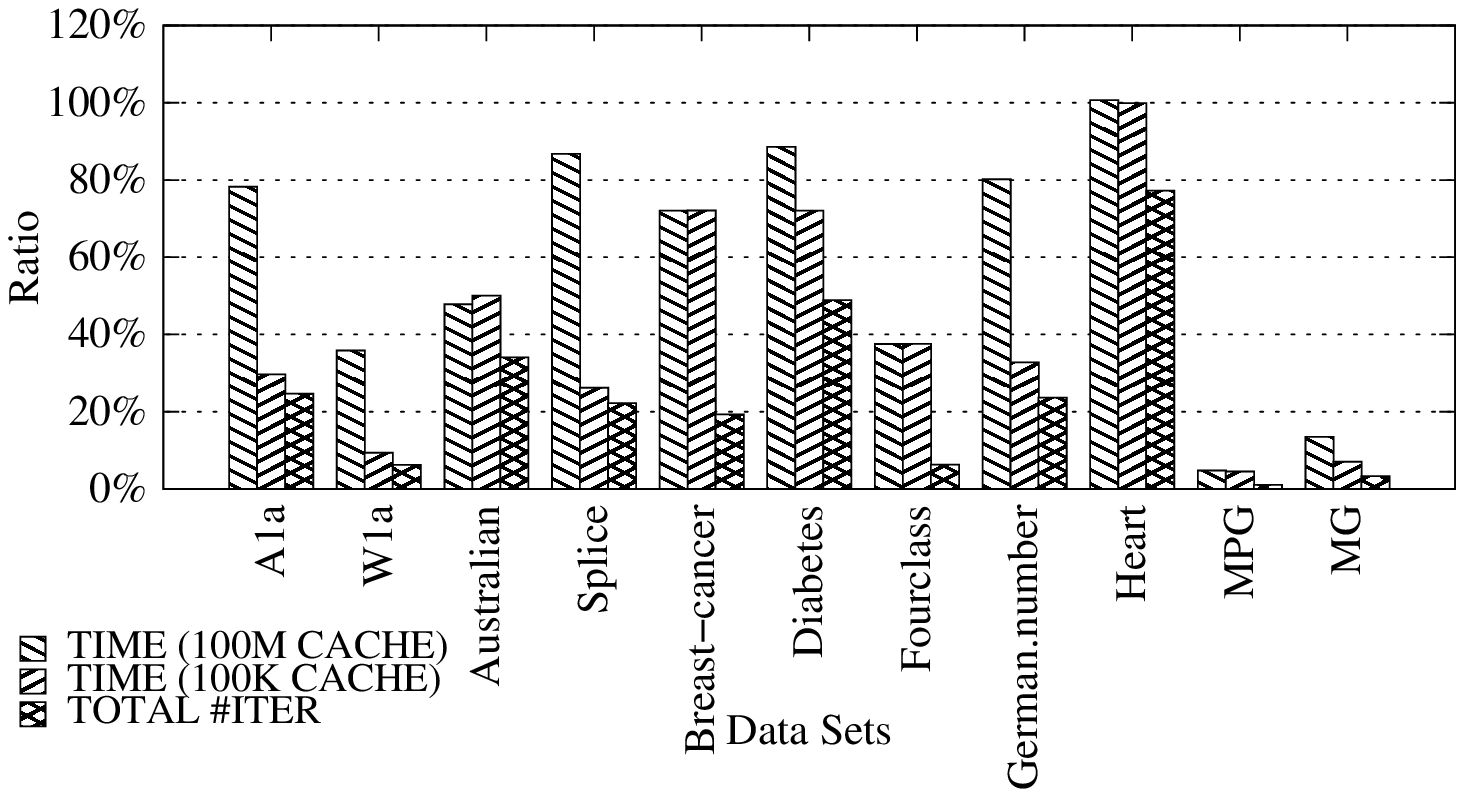}
\caption[Polynomial_ft]{Iteration and time ratios between \textit{WSS-WR} and \textit{WSS 3} using the Polynomial kernel for the "final training" step~(top: with shrinking, bottom: without shrinking).}%
\label{Polynomial ft ratio}
\end{center}
\end{figure}

\begin{figure}[!htp]
\begin{center}
\includegraphics[width=.4\textwidth,height=.15\textwidth]{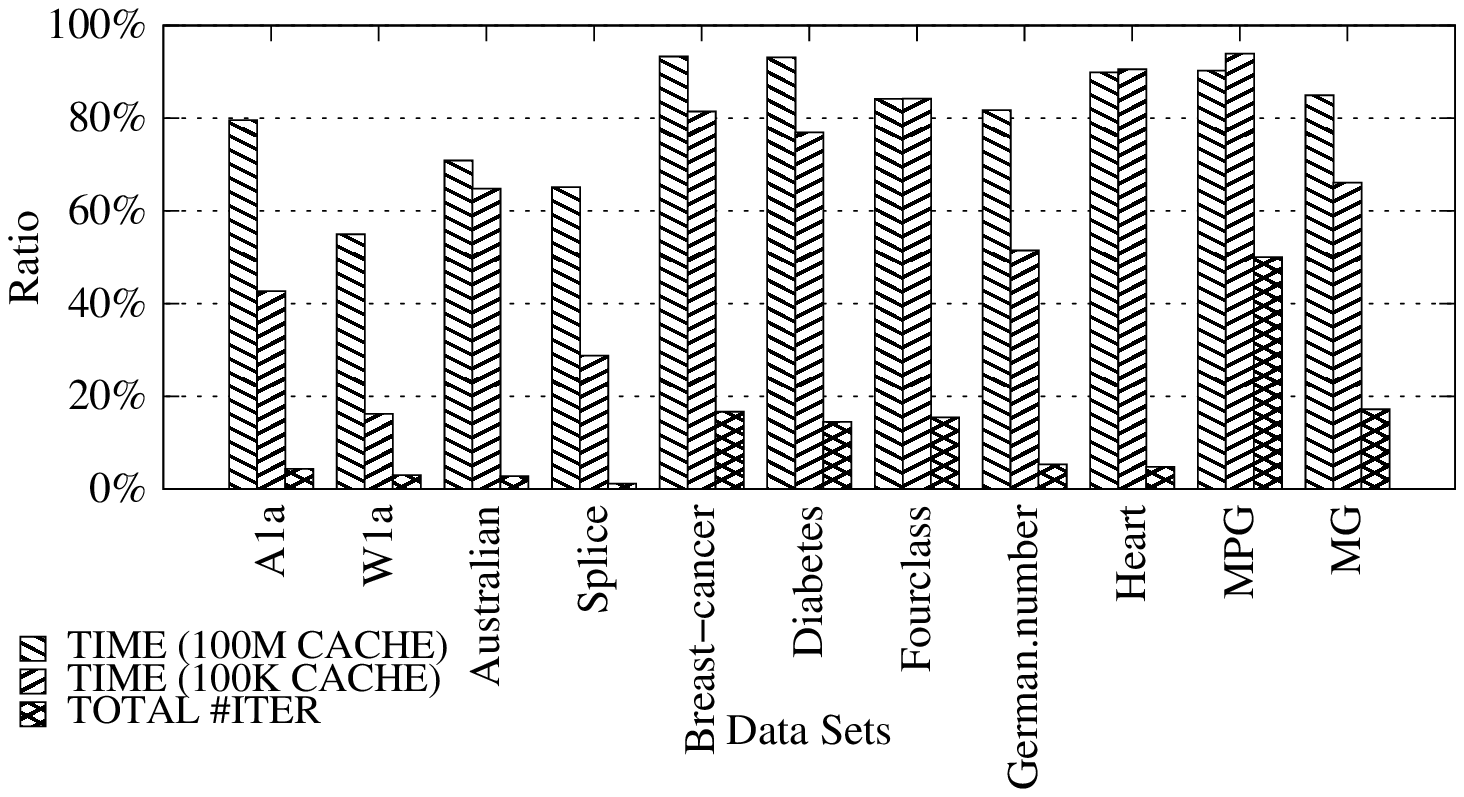} \\
 \includegraphics[width=.4\textwidth,height=.15\textwidth]{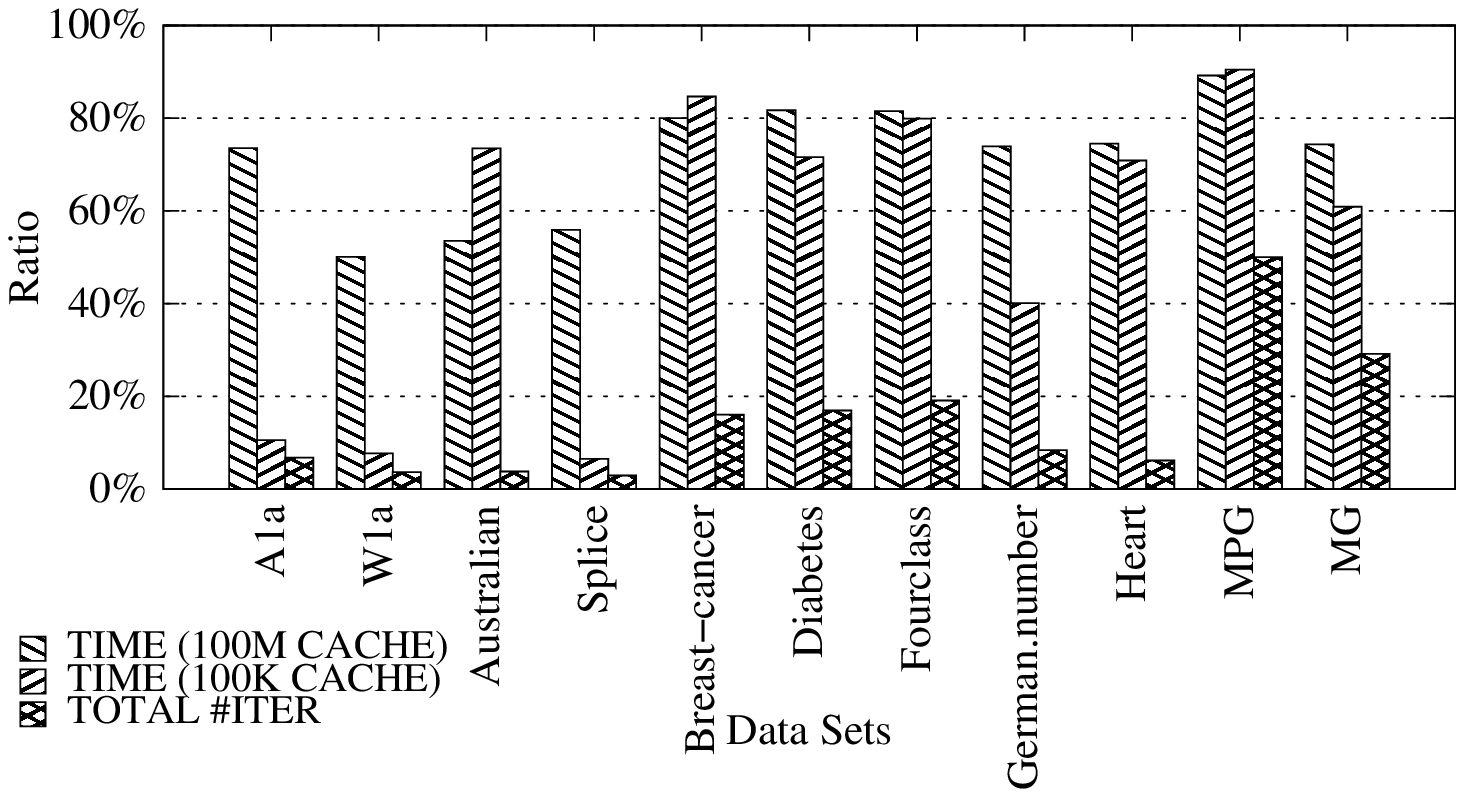}
\caption[Sigmoid_ps]{Iteration and time ratios between \textit{WSS-WR} and \textit{WSS 3} using the Sigmoid kernel for the "parameter selection" step~(top: with shrinking, bottom: without shrinking).}%
\label{Sigmoid ps ratio}
\end{center}
\end{figure}

\begin{figure}[!htp]
\begin{center}
\includegraphics[width=.4\textwidth, height=.15\textwidth]{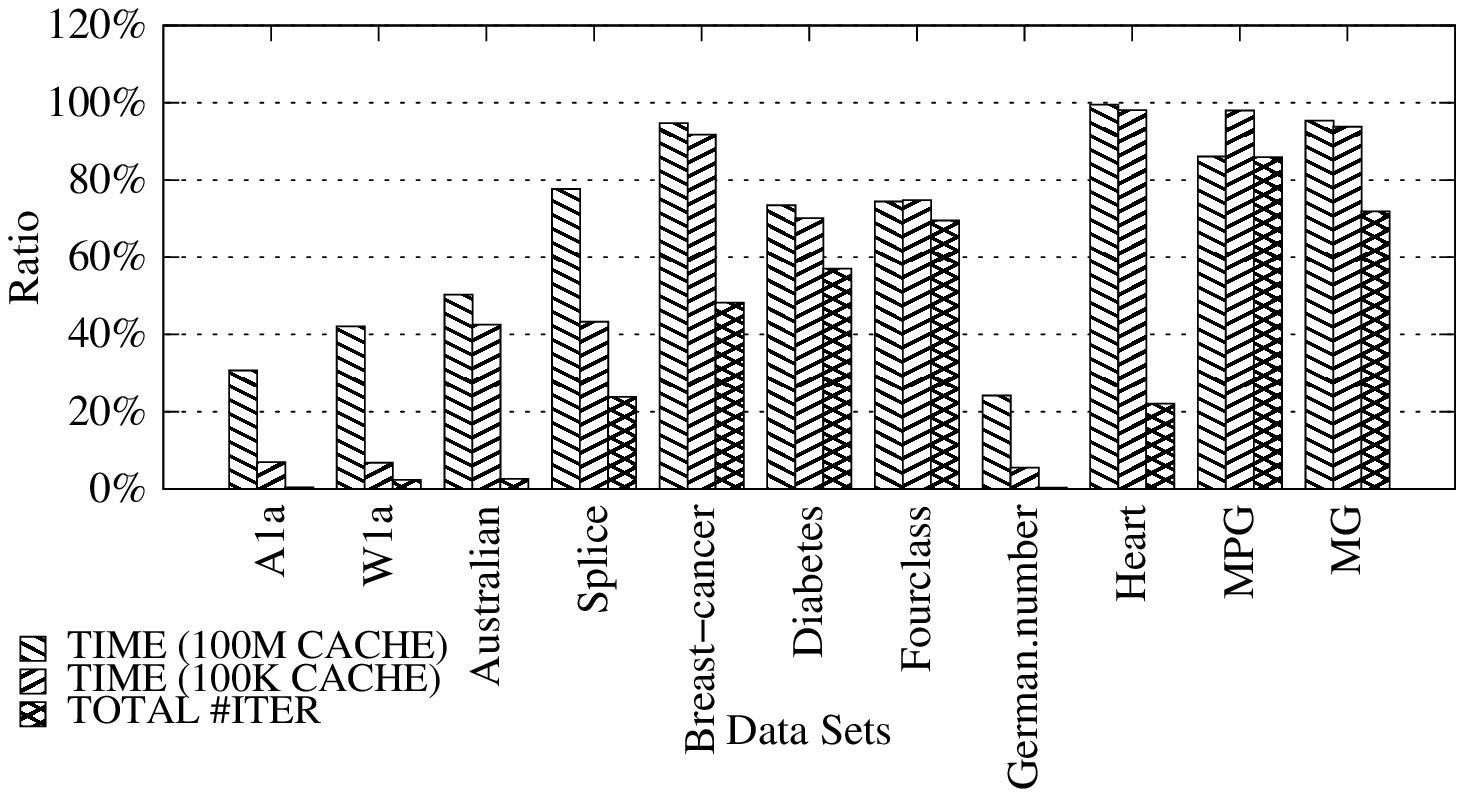} \\
\includegraphics[width=.4\textwidth, height=.15\textwidth]{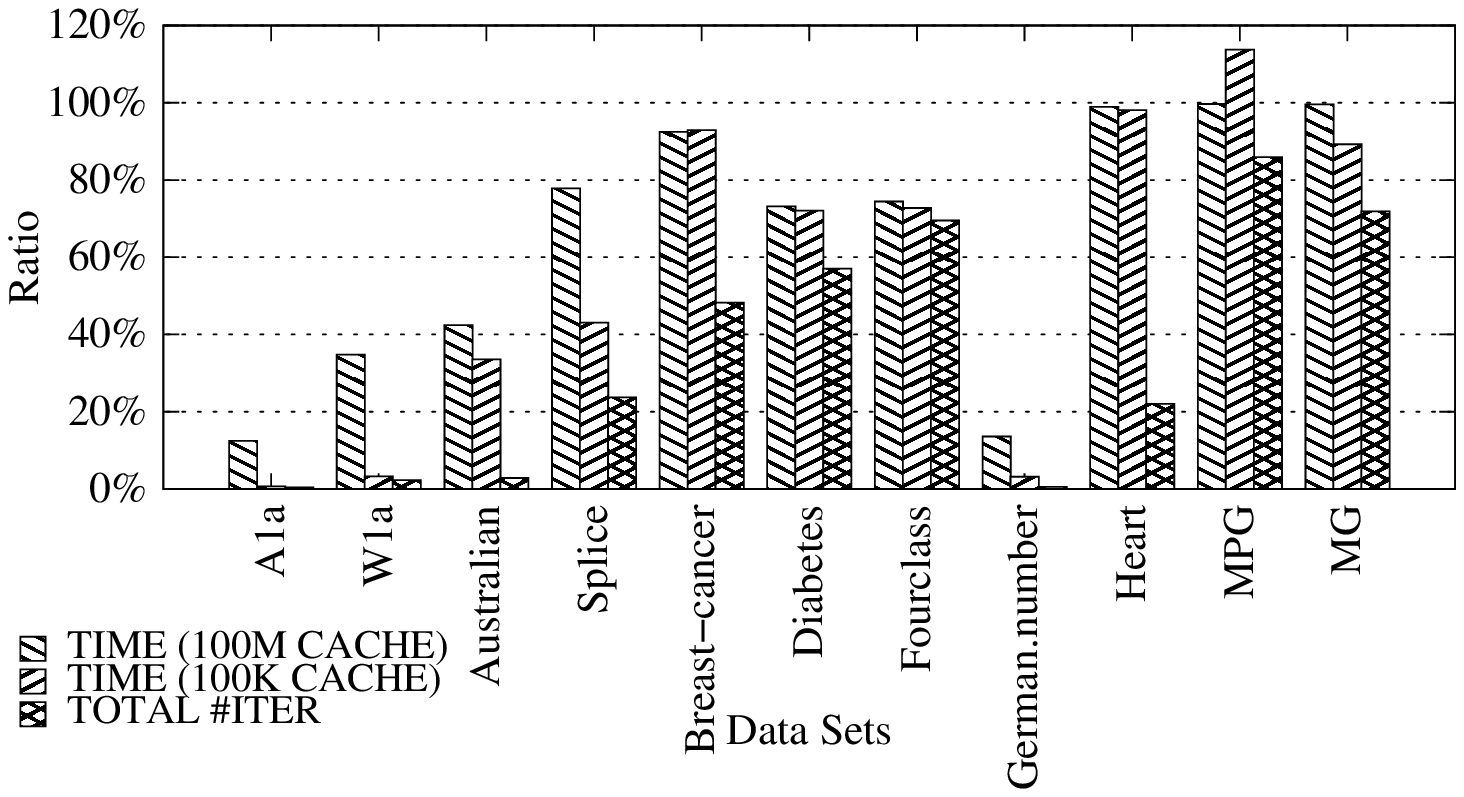}
\caption[Sigmoid_ft]{Iteration and time ratios between \textit{WSS-WR} and \textit{WSS 3} using the Sigmoid kernel for the "final training" step~(top: with shrinking, bottom: without shrinking).}%
\label{Sigmoid ft ratio}
\end{center}
\end{figure}

The number of iterations is independent of the cache size. In the
"parameter selection" step, time~(or iterations) of all parameters is
summed up before calculating the ratio. In general the "final
training" step is very fast, so the timing result may not be accurate.
Hence we repeat this step 4 times to obtain more reliable timing
values. Fig.~\ref{RBF ps ratio}-~\ref{Sigmoid ft ratio} present obtain
ratios, They are in general less than 1. We can conclude that
using \textit{WSS-WR} is in general better than using \textit{WSS 3}.

\subsubsection{Caching and Shrinking techniques in \textit{WSS-WR}}
According to the analysis of the experimental results, in
\textit{WSS-WR} model, the caching and shrinking techniques do not
have much more effects on SMO decomposition method. The data in
Table~\ref{W1a,Comparison of Caching and
  Shrinking},~\ref{German.numer,Comparison of Caching and Shrinking}
certify the above conclusion.

\begin{table}
\centering
\begin{tabular}[!htp]{|c|c|c|c|c|}
\hline \multicolumn{5}{|c|}{ \textit{WSS-WR}} \\
\hline method & RBF & Linear  & Poly. & Sigm. \\
\hline 100M,~shrinking & 56.8936 & 15.6038 & 9.3744 & 15.9825  \\
\hline 100M,~nonshrinking & 56.8936 & 15.6038 & 9.7964 & 16.1321 \\
\hline 100K,~shrinking & 59.9929 & 17.7689 & 9.4823 & 16.7169  \\
\hline 100K,~nonshrinking & 59.8346 & 15.3692 & 8.9024 & 16.8676  \\
\hline \multicolumn{5}{|c|}{\textit{WSS 3}} \\
\hline & RBF & Linear & Poly. & Sigm.  \\
\hline 100M,~shrinking &  400.7368 & 49.5343 & 34.0337 & 29.0977 \\
\hline 100M,~nonshrinking & 505.7596 & 62.5636 & 34.3321 & 32.2544 \\
\hline 100K,~shrinking & 1655.7840 & 232.5998 & 140.9247 & 103.3014 \\
\hline 100K,~nonshrinking & 1602.4966 & 533.9386 & 93.3340 & 220.2574 \\
\hline
\end{tabular}
\caption{ W1a, comparison of caching and shrinking } \label{W1a,Comparison of Caching and Shrinking}
\end{table}
 
\begin{table}
\centering
\begin{tabular}[!htp]{|c|c|c|c|c|}
\hline \multicolumn{5}{|c|}{ \textit{WSS-WR}}  \\
\hline method & RBF & Linear  & Poly. & Sigm.   \\
\hline 100M,~shrinking & 97.8465 & 24.0803 & 15.9311 & 31.7092  \\
\hline 100M,~nonshrinking & 97.8465 & 24.0803 & 15.9311 & 31.7092  \\
\hline 100K,~shrinking &  99.8329 & 26.4382 & 15.0456 & 30.8630 \\
\hline 100K,~nonshrinking & 92.9560 & 25.7199 & 15.0456 & 28.5457 \\
\hline \multicolumn{5}{|c|}{\textit{WSS 3}} \\
\hline method & RBF & Linear & Poly. & Sigm.  \\
\hline 100M,~shrinking &  299.3279 & 168.0513 & 46.5298 & 38.7958 \\
\hline 100M,~nonshrinking &  413.3258 & 764.8448 & 52.9702 & 42.9117 \\
\hline 100K,~shrinking &  1427.4005 & 505.8935 & 201.4057 & 60.0467 \\
\hline 100K,~nonshrinking & 3811.7786 & 1213.8522 & 336.8386 & 71.2662 \\
\hline
\end{tabular}
\caption{ German.numer, comparison of caching and shrinking } \label{German.numer,Comparison of Caching and Shrinking}
\end{table}

\subsubsection{Experiments of Large Classification Datasets}
Next, the experiment with large classification sets is handled by a
similar procedure. As the parameter selection is time consuming, we
adjust the "parameter selection" procedure to a 16-point searching.
The cache size are 300MB and 1MB. The experiments employ RBF and Sigmoid kernel
methods, along with that Sigmoid kernel in general leads the worst ratio between
\textit{WSS-WR} and \textit{WSS 3}. Table~\ref{large sets} gives iteration and time
ratios. Comparing the results among small problems, we safely draw the
conclusion that ratios of time and iteration in large datasets are
less than those of them in small ones, especially, in the situation of
small size of cache.

\begin{table}
\centering
\begin{tabular}[!htp]{|c|c|c|c|c|c|c|}
\hline \multicolumn{3}{|c|}{300MB cache} & \multicolumn{4}{|c|}{RBF kernel}  \\
\hline \multicolumn{3}{|c|}{Datasets } & \multicolumn{2}{|c|}{
  Shrinking} & \multicolumn{2}{|c|}{No-Shrinking}   \\
\hline Problem & \#data & \#feat. & Iter. & Time & Iter. & Time \\
\hline a9a & 32,561 & 123 & 0.0522 & 0.2306 & 0.3439 & 0.8498   \\
\hline w8a & 49,749 & 300 & 0.0370 & 0.0327 & 0.0149 & 0.2106 \\
\hline IJCNN1 & 49,990 & 22 & 0.1187 & 0.5889 & 0.3474 & 0.7422 \\
\hline \multicolumn{3}{|c|}{300MB cache} & \multicolumn{4}{|c|}{Sigmoid kernel}  \\
\hline \multicolumn{3}{|c|}{Datasets} & \multicolumn{2}{|c|}{
  Shrinking} & \multicolumn{2}{|c|}{No-Shrinking}   \\
\hline Problem & \#data & \#feat. & Iter. & Time & Iter. & Time \\
\hline a9a & 32,561 & 123 & 0.0522 & 0.2282 & 0.3439 & 0.8603  \\
\hline w8a & 49,749 & 300 & 0.0380 & 0.0443 & 0.0339 & 0.4554 \\
\hline IJCNN1 & 49,990 & 22 & 0.1309 & 0.5928 & 0.3883 & 0.8147 \\
\hline \multicolumn{3}{|c|}{1MB cache Nonshrinking} & \multicolumn{2}{|c|}{
  RBF} & \multicolumn{2}{|c|}{Sigmoid}   \\
\hline Problem & \#data & \#feat. & Iter. & Time & Iter. & Time \\
\hline a9a & 32,561 & 123 & 0.0522 & 0.0687 & 0.3439 & 0.3799  \\
\hline w8a & 49,749 & 300 & 0.0370 & 0.0365 & 0.0149 & 0.3279 \\
\hline IJCNN1 & 49,990 & 22 & 0.1187 & 0.1673 & 0.3474 & 0.4093 \\
\hline
\end{tabular}
\caption{ Large problems: Iteration and time ratios between \textit{WSS-WR} and
  \textit{WSS 3} for 16-point "parameter selection".} \label{large sets}
\end{table}


\subsection{Convergence Graph of \textit{WSS-WR}}
For the reason that \textit{WSS-WR} employs \textit{Algorithm 2} and second order
information method, which \textit{WSS 3} used, we just compare the
convergence rates between them.
 
We set $C=1,\gamma=0,\epsilon=10^{-3}$
during the whole comparison.

First, the evaluation is made on the following datasets: a1a, w1a,
australian, splice, breast-cancer, diabetes, fourclass, german.numer
and heart.

The illustrations in Fig.~\ref{convergence} are the first step evaluation, we compare the
convergence ratios on nine datasets between \textit{WSS-WR} and
\textit{WSS 3} with RBF kernel. Other charts are omitted here for short.

For further analysis, we choose two datasets: w1a and breast-cancer to
made evaluations by using diversity kernel methods~(Linear, RBF,
Polynomial, Sigmoid). Fig.~\ref{w1a_convergence_kernel} compare the
convergence on W1a  between \textit{WSS-WR} and \textit{WSS 3}. The
illustrations of breast-cancer are omitted here for short.

\begin{figure}[!htp]
\includegraphics[scale=.2]{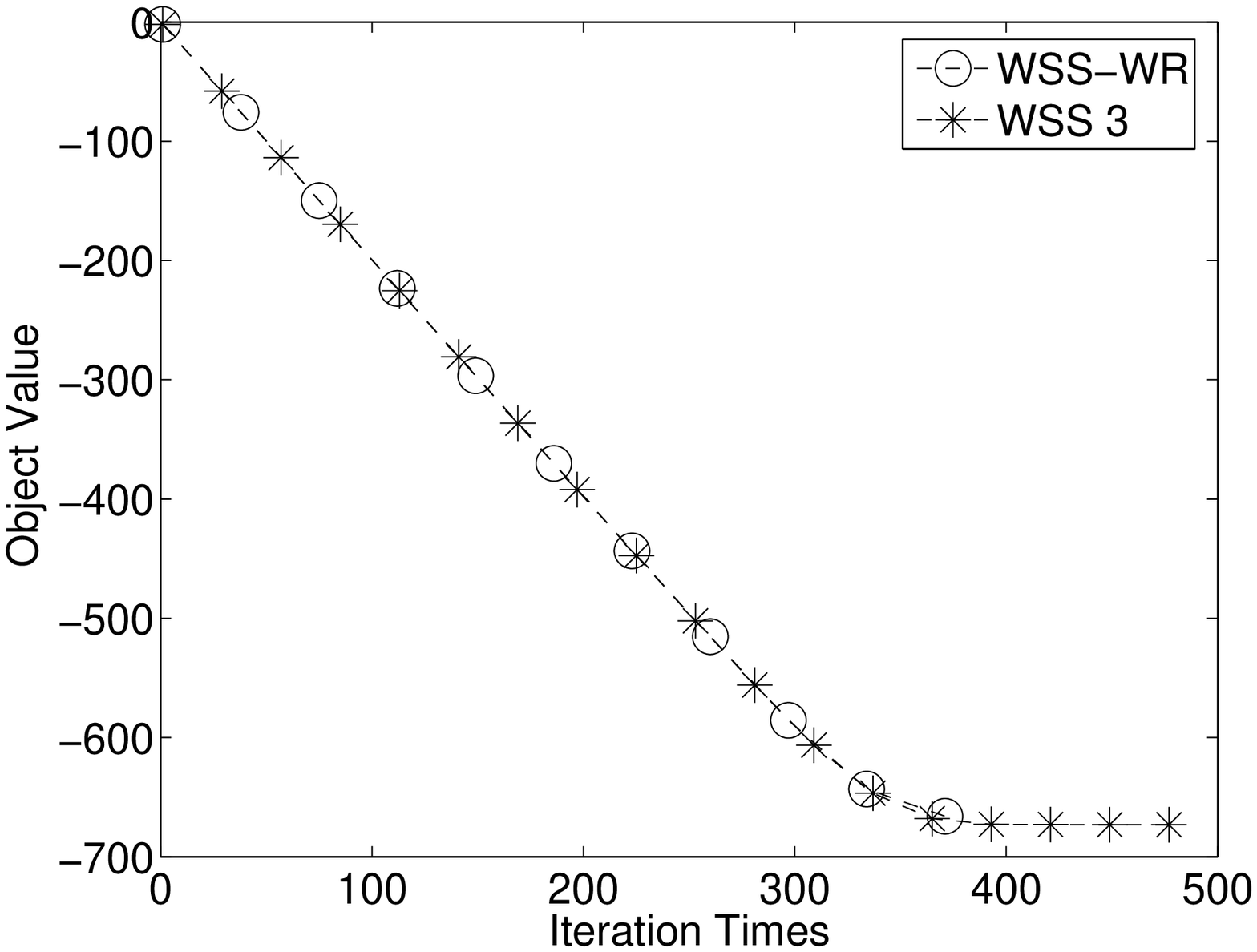}\hfill
\includegraphics[scale=.2]{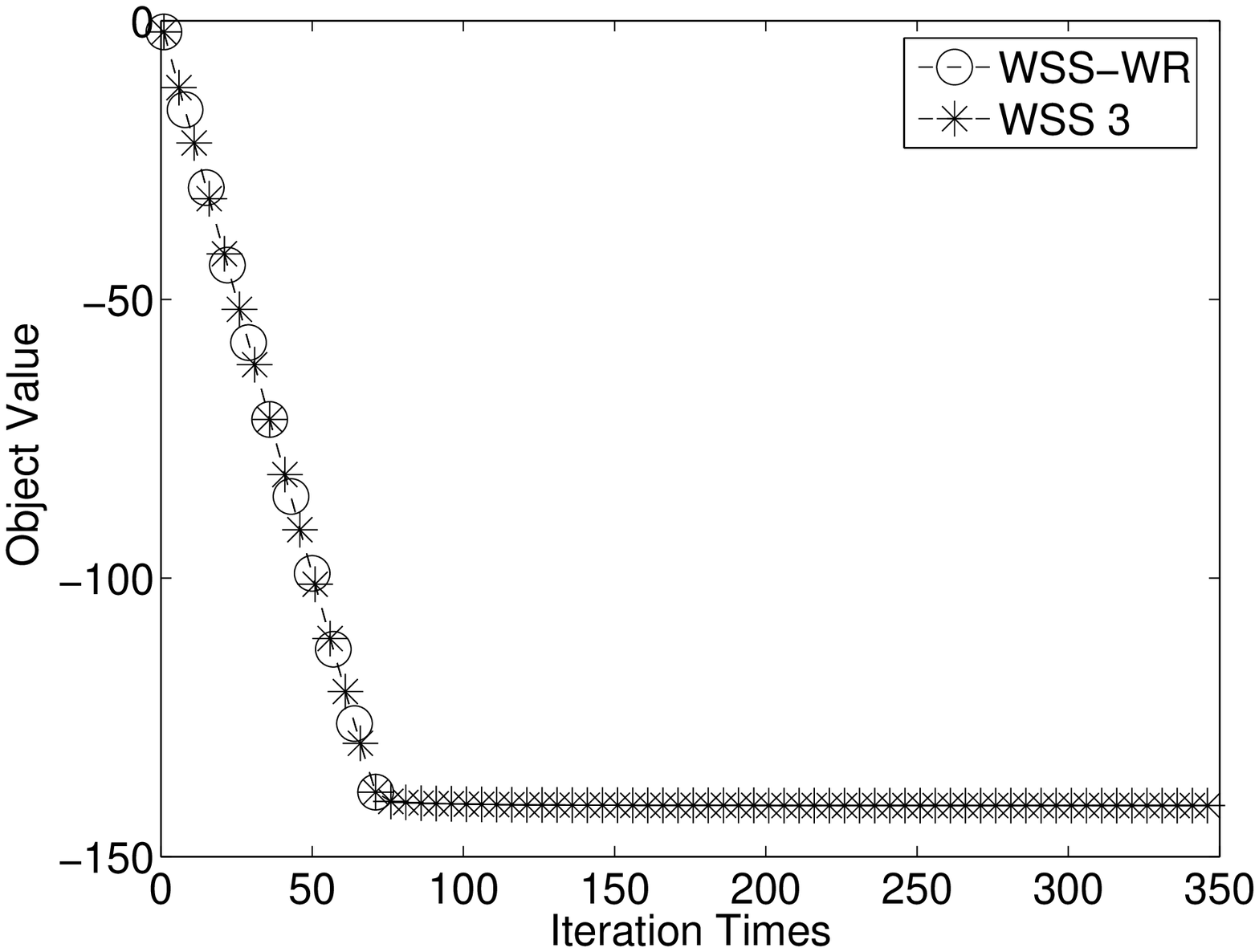}
\\
\includegraphics[scale=.2]{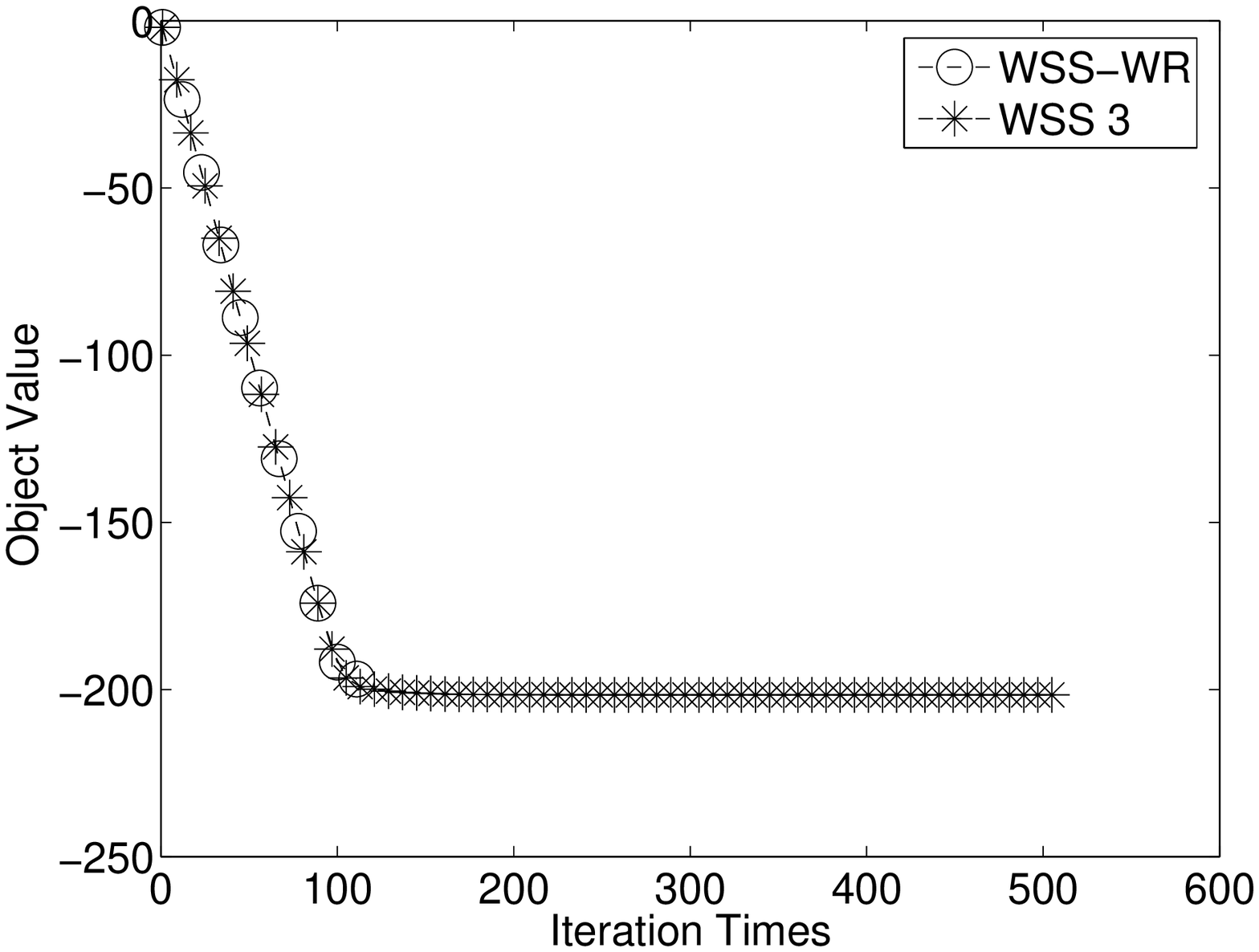}
\hfill
\includegraphics[scale=.2]{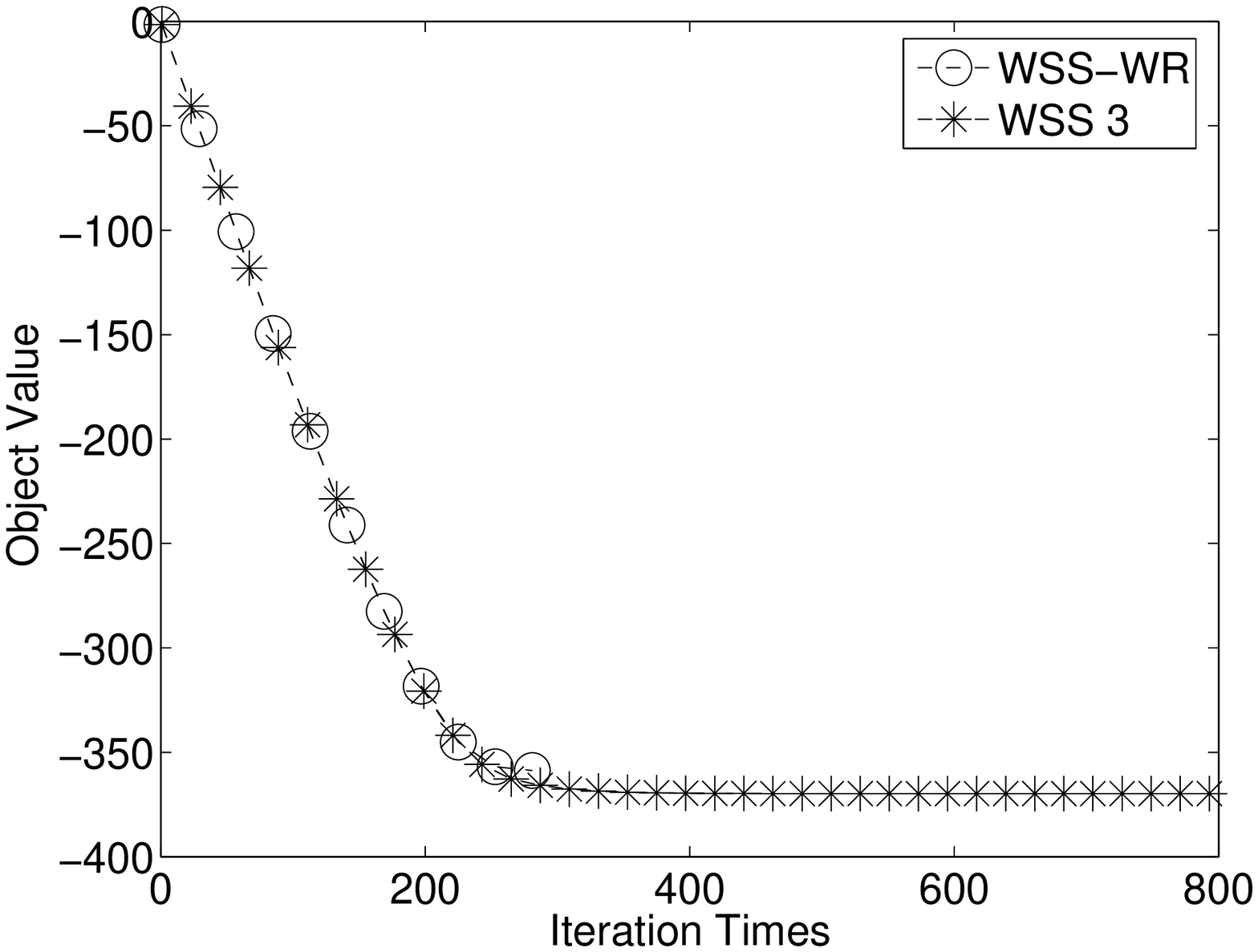} \\
\caption[a1a_convergence]{The comparison of convergence on several
  datasets between \textit{WSS-WR} and \textit{WSS 3} with RBF kernel.
  ~(Datasets in order are: a1a, w1a, australian, splice) }
\label{convergence}
\end{figure}


\begin{figure}[!htp]
\includegraphics[scale=.2]{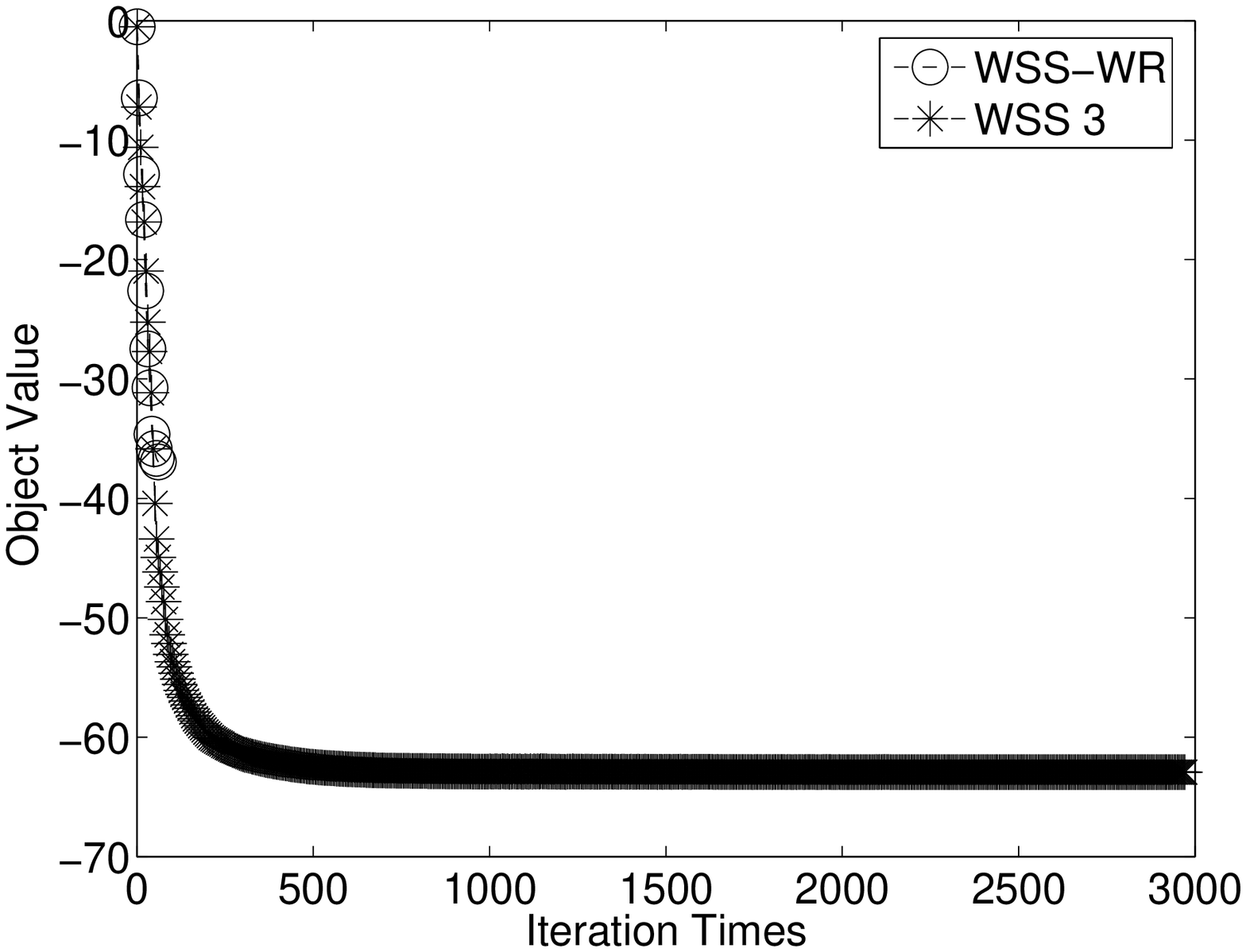} \hfill
\includegraphics[scale=.2]{pic/convergence/w1a_convergence.eps}
\\
\includegraphics[scale=.2]{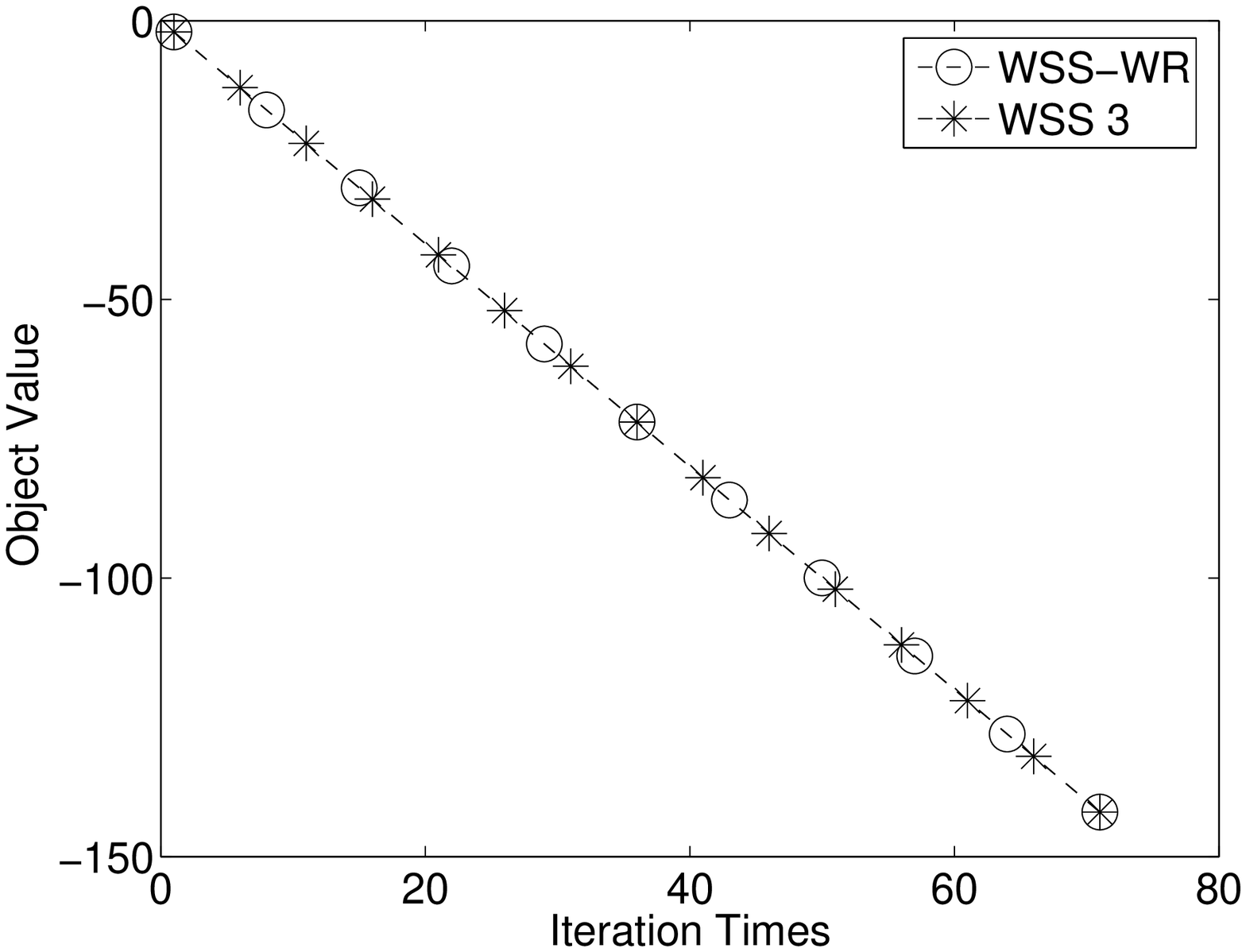}
\hfill
\includegraphics[scale=.2]{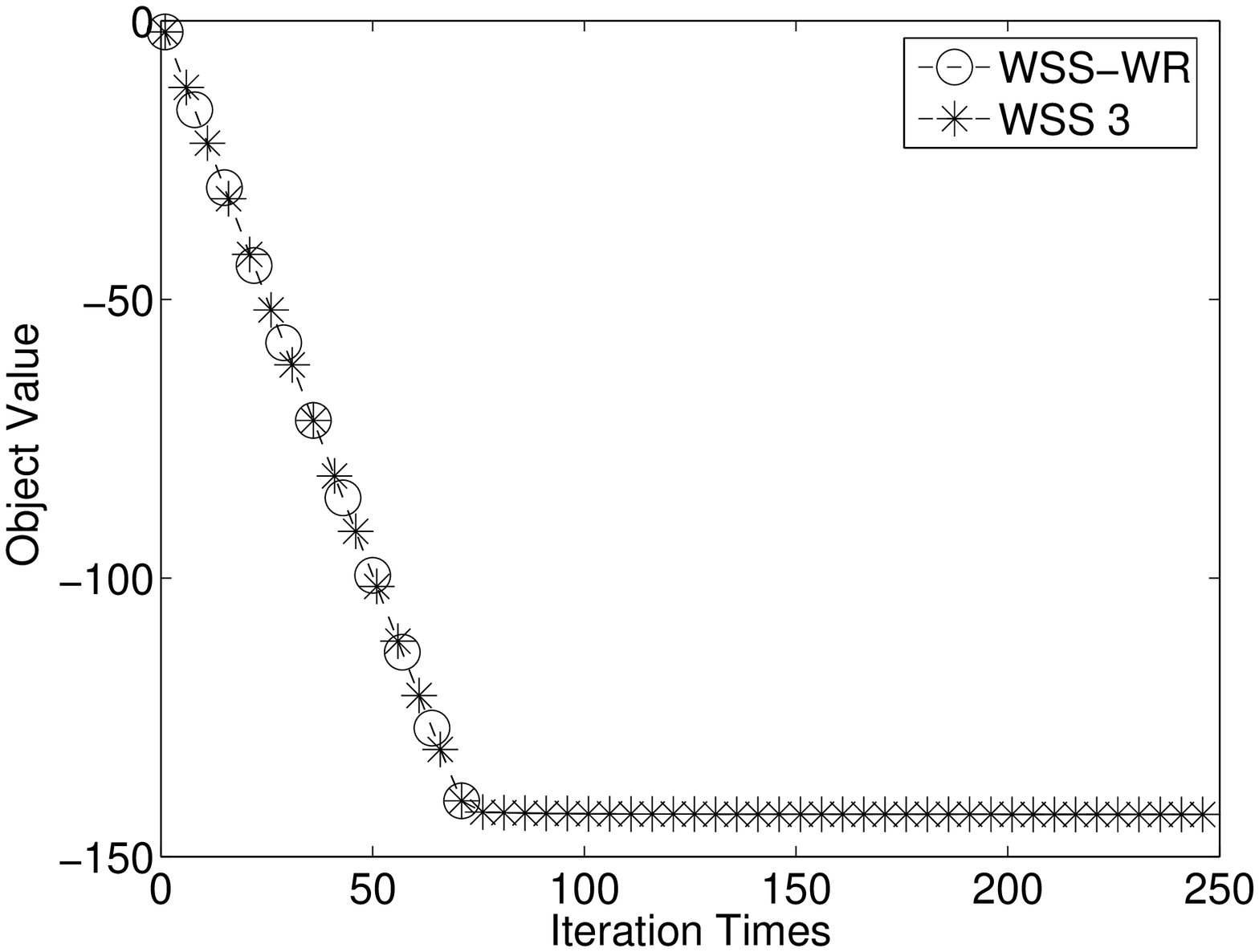}
\caption[w1a_convergence_kernel]{The comparison of convergence on W1a
  between \textit{WSS-WR} and \textit{WSS 3}, using diversity kernels:
  RBF, Linear, Polynomial, Sigmoid}%
\label{w1a_convergence_kernel}
\end{figure}

From the Fig.~\ref{convergence},~\ref{w1a_convergence_kernel}. The
convergence rates of \textit{WSS-WR} and \textit{WSS 3} are exactly
the same at the beginning of the procedure. In addition,
\textit{WSS-WR} will be terminated soon when it reaches the optimum,
while \textit{WSS 3}, on the contrary, will hold the objective value
or make insignificant progress with much more time consumed. Thus, it
is reasonable to conclude that \textit{WSS-WR} is much more efficient.

\subsection{Discussion}
With the above analysis, our main observations and conclusions from
Fig.~\ref{RBF ps ratio}-\ref{Sigmoid ft ratio} and Table~\ref{W1a,Comparison of Caching and Shrinking},~\ref{German.numer,Comparison of Caching and Shrinking},~\ref{large sets} are in the
following:
\begin{enumerate}
\item The improvement of reducing the number of iterations is significant,
  which can be obtained from the illustrations.

\item Using \textit{WSS-WR} dramatically reduces the cost of time. The reduction is more
dramatic for the parameter selection step, where some points have low convergence rates.

\item \textit{WSS-WR} outperforms \textit{WSS 3} in most datasets, both in the
  "parameter selection" step and in "final training" steps. Unlike \textit{WSS 3},
  the training time of \textit{WSS-WR} does not increase when the amount of
  memory for caching drops. This property indicates that \textit{WSS-WR} is useful
  under such situation where the datasets is too large for the kernel
  matrices to be stored, or where there is not enough memory.


\item The shrinking technique of LIBSVM~\cite{Chang.2001} was
  introduced by Pai-Hsuen Chen \textit{et al.},~\cite{Chen.2006} to
  make the decomposition method faster. But in the view of Table~\ref{W1a,Comparison of Caching and Shrinking},
 ~\ref{German.numer,Comparison of Caching and Shrinking}, experiments
  of regression and large classification problems, shrinking
  technique almost does not shorten the training time by using \textit{WSS-WR}.

\item 
Fig.~\ref{RBF ps ratio}-\ref{Sigmoid ft ratio} indicate that the
relationship of the five ratios can be described as follows:
\[
ratio5 < ratio4 < ratio3 < ratio2 < ratio1
\]
Though this may not be valid in all datasets.
\end{enumerate}

\section{Conclusions}
By analyzing the available working set selection methods and some
interesting phenomena, we have proposed a new working set selection
model--Working Set Selection Without Reselection~(\textit{WSS-WR}).
Subsequently, full-scale experiments were given to
demonstrate that \textit{WSS-WR} outperforms \textit{WSS 3} in almost datasets during
the "parameters selection" and "final training" step. Then, we
discussed some features of our new model, by analyzing the results of
experiments. A theoretical study on the convergence of \textit{WSS-WR}
and to continually improve this model are our future work.
\section{Acknowledgements}
This work is under the support of the National High
Technology Research and Development Program of China~(863 Program)
under Grant NO.2006AA01Z232. Thanks to Chih-Chung Chang and Chih-Jen
Lin for their powerful software, LIBSVM.

\bibliographystyle{IEEEtran}
\bibliography{WSS_WR_ICTAI.bib}
\end{document}